\documentclass{article}

\usepackage[english]{babel}

\usepackage[a4paper,top=2cm,bottom=2cm,left=3cm,right=3cm,marginparwidth=1.75cm]{geometry}

\usepackage{amsmath,amssymb,amsthm}
\usepackage{graphicx}
\usepackage[colorlinks=true, allcolors=blue]{hyperref}
\usepackage{verbatim}
\usepackage{cleveref}
\usepackage{amsmath}
\usepackage{algorithm}
\usepackage{algpseudocode}
\usepackage{booktabs}
\usepackage{tabularx}
\usepackage{natbib}
\usepackage{xcolor}
\usepackage{subcaption}
\usepackage[percent]{overpic}

\usepackage{multirow}
\usepackage[title]{appendix}

\usepackage{tikz}
\usetikzlibrary{arrows.meta,positioning}
\usepackage{tcolorbox}
\usepackage{authblk}

\usepackage{endnotes}
\usepackage{postnotes}
\postnotesetup{backlink}

\newtheorem{proposition}{Proposition}
\newtheorem{definition}{Definition}
\newtheorem*{remark}{Remark}
\newtheorem{corollary}{Corollary}

\newcommand{\tr}{tr}
\newcommand{\beatbase}[1]{{#1}$^\dagger$}
\title{Decision-Making under Combinatorial Risk}
\author[1]{Yifan Hong\thanks{Email: \texttt{hongyf23@mails.tsinghua.edu.cn}}}
\author[1]{Hongmiao Fan}
\author[1]{Chen Wang\thanks{Corresponding author. Email: \texttt{chenwang@tsinghua.edu.cn}}}
\affil[1]{Department of Industrial Engineering, Tsinghua University}
\date{}

\begin{document}
\maketitle

\begin{abstract}
Decision-making under risk is typically studied through single-shot lottery choices. Yet many real decisions involve \emph{combinatorial risk}, where risk arises from multiple risky components, so the lottery over outcomes is induced rather than given outright and can be costly to evaluate exactly.
We introduce an investment-allocation task to study decision under combinatorial risk, where investing in a component raises its success probability and thereby reshapes the outcome distribution.
Participants favor the option with the larger probability increment, and, when increments are equal, the option with the higher initial success probability. Revealing the induced probability mass function (PMF) substantially changes behavior, making participants less responsive to combinatorial-risk features and reducing choice variance.
To explain these patterns, we move beyond standard benchmarks and hand-crafted hypotheses with symbolic regression to discover compact descriptive models.
The discovered models rely mainly on combinatorial-risk features, such as the after-investment success probability, rather than exact evaluation of the full induced distribution.
Behavior under the displayed PMF is then well explained by augmenting this model with a prospect-theoretic residual model.
The results show that people navigate combinatorial risk primarily through its core features, shifting toward lottery valuation only when the induced PMF is displayed.
\end{abstract}

% \tableofcontents

\section{Introduction}
Decision-making under risk is a central topic in behavioral science, economics, and operations research. Much of the literature studies lottery choices, where each prospect is described as a known probability distribution over outcomes.
Yet many real-world decisions are not readily presented as lotteries.
In an important class of problems, the outcome arises from multiple risky components whose joint consequences determine the payoff, and no single lottery is given in advance.

Consider the allocation of a limited medical budget between two interventions: ventilators or medicines.
Ventilators improve survival for severely ill patients but do not help patients with low-to-moderate symptoms; medicines reduce mortality for low-to-moderate patients but do not save severely ill patients. The outcome of interest is the total number of survivors.
Although each decision induces a distribution over possible survival counts, that distribution is not directly presented at the time of choice.
Do decision makers represent the problem as induced lotteries, or do they rely on different strategies tailored to the combinatorial structure of the problem?

We refer to this class of problems as \emph{combinatorial risk}: risk arises from the combination of multiple risky components.
We use a simple investment-allocation task as a canonical instance, and adopt a commercial framing to avoid the moral considerations that could confound choices in a medical context.
Participants act as a supplier deciding which of two customers to target with a promotional investment.
Each customer has an initial probability of purchase, the promotion increases the purchase probability of the targeted customer, and the participant's payoff depends on the total number of purchases across the two customers. 

We find that people prefer investments with higher expected value and lower variance, even though these quantities are only implicit in the combinatorial-risk features.
The expected value of an investment is governed by its probability increment, whereas its variance also depends on the initial success probabilities.
Participants respond systematically to both: when one option offers a larger increment, they tend to choose it; when the increments are equal, they tend to choose the option with the higher initial, and therefore after-investment, success probability.
An interesting question is how providing the induced lotteries changes behavior. 
We compare a control condition, in which participants see only the combinatorial-risk features, with a treatment condition, in which they additionally see the induced PMFs over total successes.
Revealing the PMF substantially changes choice patterns, yet it does not make decisions more advantageous in terms of expected value.
Participants given the PMFs become less responsive to the combinatorial-risk features, and the way their choices respond to payoff magnitude also differs from the control condition.

We develop descriptive models to understand the behavioral patterns. Prospect-theoretic models on the induced lotteries achieve strong predictive performance, but they are behaviorally implausible when the PMF is not readily available.
To move beyond a small set of hand-crafted theories, we employ a symbolic regression (SR) method to discover descriptive models. 
The method combines evolutionary search with LLM-based generation, and organizes discovered expressions into an ontology that guides exploration. % of recurring features, functional forms, and concepts 
The discovered models show that behavior is organized primarily around salient quantities: after-investment success probabilities and probability increments.
In the control condition, the best models rely on these combinatorial-risk features without exact expected-payoff evaluation.
In the treatment condition, the best models combine these features with nonlinear transformations of induced-lottery features.
A residual analysis further shows that the control-to-treatment shift can be captured by layering a prospect-theoretic evaluation of the displayed PMF on top of the control model.
Taken together, these results indicate that people do not represent combinatorial risk as an explicit lottery. Instead, they reason over salient probabilistic features, and utilize distributional information as an added valuation layer when available.

The paper proceeds as follows.
\Cref{sec:related-work} reviews decision under risk and symbolic regression. 
\Cref{sec:formulation} formally introduces the investment-allocation task and describes the experiment design.
\Cref{sec:formulation-behavioral-finding} presents the behavioral findings.
We then proceed to descriptive modeling, with
\Cref{sec:benchmarks} presenting benchmark models and
\Cref{sec:experiments} using symbolic regression to systematically search for symbolic models, and
\Cref{sec:residual-analysis} explaining the impact of PMFs via residual analysis.
\Cref{sec:discussion} discusses the findings, limitations, and future work, and \Cref{sec:conclusion} concludes the paper.

\section{Related work}\label{sec:related-work}
\subsection{Decision-making under Risk} Decision-making under risk involves a choice between lotteries, each defined by a known probability distribution over outcomes. A lottery $A$, denoted \( (p_A,x_A), \) yields outcome $x_{Ai}$ with probability $p_{Ai}$ for $i=1,2,...,n$. Consider the choice between $A$ and $B$. The long-run optimal policy for repeated play is to maximize \emph{expected-value} \( \mathbb{E}_{p}[x] = p^\top x \).
However, human choices deviate from expected-value maximization. A canonical illustration is the St.\ Petersburg paradox: people will not pay arbitrarily much for a gamble with unbounded expected value. This motivates separating objective outcomes from subjective value \citep{parmigiani2009decision}. The resulting normative framework is \emph{expected utility} theory, where rational choice is characterized by maximizing expected utility, \( \mathbb{E}_{p}\left[u(x)\right] = \sum_{i=1}^n p_{i}\cdot u(x_{i}), \) where $u(\cdot)$ is the utility function, applied element-wise on each outcome. Researchers have since proposed families of utility functions to capture risk attitudes and preferences \citep{pratt1964risk}.
Prospect theory (PT) highlights further systematic deviations from expected utility and argues that people transform decision problems through mental editing, evaluate outcomes relative to a reference point, and apply \emph{decision weights} rather than objective probabilities \citep{Kahneman1979prospect}. For example, small probability-events are over-weighted in decisions \citep{burns2010overweighting}. Later work continues to document anomalies and has developed increasingly accurate descriptive models \citep{erev2017anomalies,peterson2021using}.

\subsection{Symbolic Regression}
Symbolic regression aims to recover an interpretable symbolic expression from data. 
Given a dataset $\mathcal{D}=\{X,y\}$ and a set of primitives, SR searches for an expression $f\in\mathcal{T}$ that fits the data well while remaining simple enough to interpret. 
This trade-off is often formulated either as multi-objective optimization over predictive fit and complexity, or as empirical risk minimization under an explicit complexity constraint:
\[
f^\star \in \arg\min_{f\in \mathcal{T}_D}\; \ell(f,\mathcal{D}),
\]
where $\mathcal{T}_D$ denotes the set of expression trees of bounded depth $D$. 
The problem is computationally challenging: SR is NP-hard \citep{virgolinsymbolic2022}, and the number of candidate expressions grows exponentially with the tree depth \citep{kim2023learning}.

A large body of research focuses on \textbf{search-based SR}, where the central challenge is to navigate the combinatorial expression space efficiently. Classical approaches represent expressions as trees and use genetic programming (GP) to evolve candidate formulas through selection, crossover, and mutation. Despite the rise of newer paradigms, state-of-the-art GP methods remain highly competitive on SR benchmarks \citep{burlacu2020operon,cranmer2023interpretable}.
%More recent work also considers richer representations and decomposition-based search; for example, \citet{kahlmeyer2025dimension} use directed acyclic graphs with functional dependence tests to reduce problem dimensionality and guide beam search.

More recently, \textbf{generative SR} treats SR as a sequence generation problem. These methods train neural or reinforcement-learning models on existing or synthesized formulas so that plausible expressions can be produced directly at inference time. Examples include autoregressive RNNs \citep{petersen2021deep} and Transformer-based models \citep{kamienny2022end}. Hybrid approaches further combine generation with search, for example by using GP to refine generated expressions \citep{holt2023deep,ying2025neural}. Generative SR can achieve competitive accuracy with substantially faster inference, but its effectiveness often depends on pretraining over large corpora of real or synthetic expressions.
%, while KAN-based approaches exploit compositions of univariate functions for feature pruning and symbolic fitting \citep{liu2025kan,liu2024kan2}.
% Related work also uses learned univariate shape functions to improve interpretability beyond closed-form equations \citep{kacprzyk2024shape}.

Finally, \textbf{LLM-assisted SR} extends generative SR by leveraging the broad mathematical and semantic priors of large language models. 
Beyond generating candidate expressions, LLMs can reason about the qualitative structure of formulas and suggest useful abstractions or operators. 
For example, they have been used to iteratively improve candidate expressions from feedback \citep{merler2024context}, extract semantic concepts from high-performing formulas to guide GP \citep{grayeli2024symbolic}, and propose new operands for reinforcement-learning-based SR \citep{guo2025sr}. 
These developments suggest that modern SR is evolving from pure combinatorial search toward a hybrid paradigm that combines search, neural inductive bias, and semantic prior knowledge.
\section{Decision under Combinatorial Risk}\label{sec:formulation}

We formalize a simple investment-allocation problem as a canonical instance of decision under combinatorial risk.
In this problem, each action modifies a component-level success probability, and the distribution over outcomes is induced by the combination of multiple risky components.
We characterize how these combinatorial-risk features determine the expected values and variances of the induced lotteries, and then describe the experimental design used to collect behavioral data.

\subsection{Problem Description}\label{sec:formulation-description}

\begin{figure}[htbp]
    \centering
    \includegraphics[width=1\linewidth]{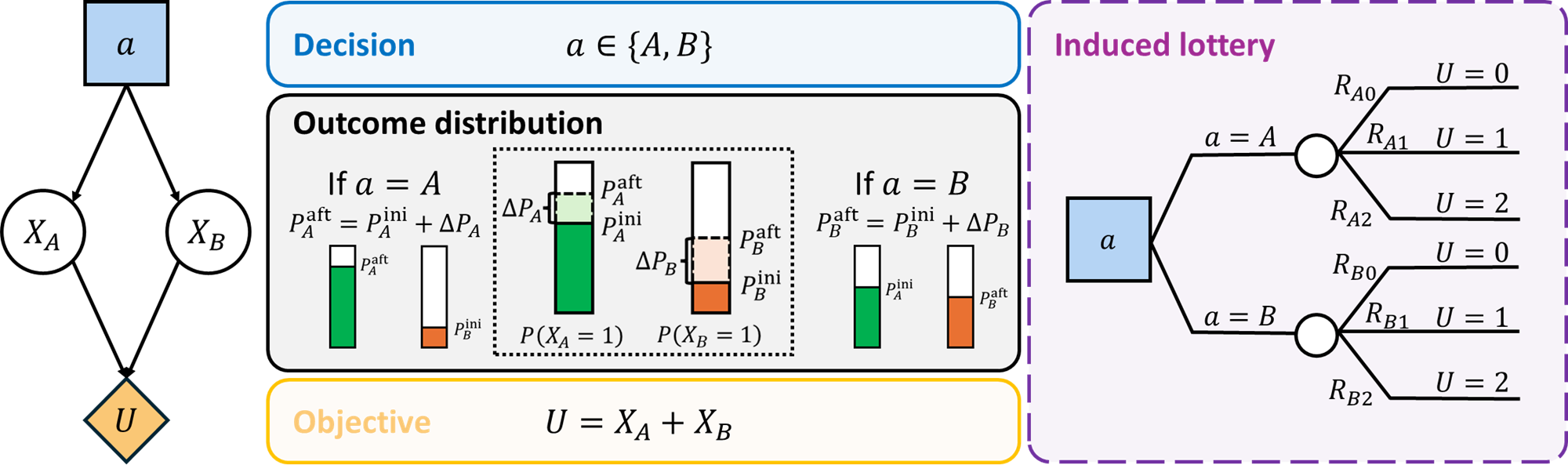}
    \caption{Influence diagram of the two-component investment-allocation problem. The decision selects one component to invest in, increases its success probability, and induces a lottery over the total number of successes.}
    \label{fig:problem-description}
\end{figure}

We formalize a two-component investment-allocation problem as a minimal instance of decision under combinatorial risk 
(see \Cref{fig:problem-description}; notation is summarized in \Cref{tab:notation-investment-problem}). 
Consider two independent Bernoulli components, denoted \(A\) and \(B\). 
Before the decision, their success probabilities are \(P_A^{\mathrm{ini}}\) and \(P_B^{\mathrm{ini}}\). 
The decision maker can make a single, indivisible investment in one of the two components, with the objective of maximizing the total number of successes $U=X_A+X_B$. 
Investing in \(A\) increases its success probability by \(\Delta P_A\) to \(P_A^{\mathrm{aft}} = P_A^{\mathrm{ini}} + \Delta P_A\), whereas investing in \(B\) increases its success probability by \(\Delta P_B\) to \(P_B^{\mathrm{aft}} = P_B^{\mathrm{ini}} + \Delta P_B\). The unchosen prospect remains unchanged. 
Throughout the paper, we refer to the random variables ($X_A$ and $X_B$) as \emph{components}, while the choice of investing in a component as an \textit{option}. Here, each component corresponds to a Bernoulli distribution, while each \textit{option} induces a Binomial distribution.

\subsection{Induced Lotteries}\label{sec:formulation-property}

Each investment choice induces a lottery, a distribution over the total successes from two Bernoulli trials, with support \(\{0,1,2\}\). 
Let \((R_{A0}, R_{A1}, R_{A2})\) denote the probability mass function over \(\{0,1,2\}\) when the investment is allocated to \(A\), and let \((R_{B0}, R_{B1}, R_{B2})\) denote the analogous distribution when the investment is allocated to \(B\).
These \emph{induced lotteries} provide an explicit representation of risk over outcomes.
Denote the induced lotteries as $(S_A, S_B)$, where $S_i$ stands for the outcome distribution if the decision is $a=i\in\{A,B\}$.

The expected value of the induced lotteries are
\(
\mathbb{E}[S_A]
=
P_A^{\mathrm{ini}}+P_B^{\mathrm{ini}}+\Delta P_A,
\,
\mathbb{E}[S_B]
=
P_A^{\mathrm{ini}}+P_B^{\mathrm{ini}}+\Delta P_B .
\)
The difference in expected value is determined by the difference in probability increment
\[
\mathbb{E}[S_A]-\mathbb{E}[S_B]
=
\Delta P_A-\Delta P_B .
\]
When an option has higher probability increment, choosing it is equivalent to choosing the induced lottery with the higher expected value.
The difference in the variances of the induced lotteries is
\[
\operatorname{Var}(S_A)-\operatorname{Var}(S_B)
=
\Delta P_A(1-2P_A^{\mathrm{ini}}-\Delta P_A)
-
\Delta P_B(1-2P_B^{\mathrm{ini}}-\Delta P_B).
\]
When the two options have the same increment \(\Delta P_A=\Delta P_B=\Delta P\), their induced lotteries have the same expected value, and the variance difference simplifies to
\[
\operatorname{Var}(S_A)-\operatorname{Var}(S_B)
=
2\Delta P\bigl(P_B^{\mathrm{ini}}-P_A^{\mathrm{ini}}\bigr).
\]
Thus, when equal probability increments offer same expected values, choosing the option with the higher initial success probability is equivalent to choosing the induced lottery with lower variance. 

%     Are choices primarily driven by local probability gains \((\Delta P_A,\Delta P_B)\), by post-investment levels \((P_A^{\mathrm{aft}},P_B^{\mathrm{aft}})\), or by distributional features of the induced PMFs, such as \((R_{A0},R_{A1},R_{A2})\) and \((R_{B0},R_{B1},R_{B2})\), or tail probabilities? 

\subsection{Experiment Design}
The experiment was designed to measure how people make decisions under combinatorial risk and how their choices change when the induced lottery is provided.
We implemented the investment-allocation problem using a vendor-marketing framing. The framing preserves the structure of the problem while avoiding moral considerations that could confound decisions in medical contexts.

On each trial, participants acted as a supplier deciding which of two customers, \(A\) or \(B\), to target with a promotion. Each customer had an initial probability of making a purchase, and the promotion increased the purchase probability of the targeted customer. 
Each successful purchase yielded a fixed payoff, and the participant's payoff depended on the total number of purchases across the two customers.
An example problem from the treatment condition is shown in \Cref{fig:experiment-problem-example}.

\begin{figure}[htbp]
    \centering
    \includegraphics[width=0.7\linewidth]{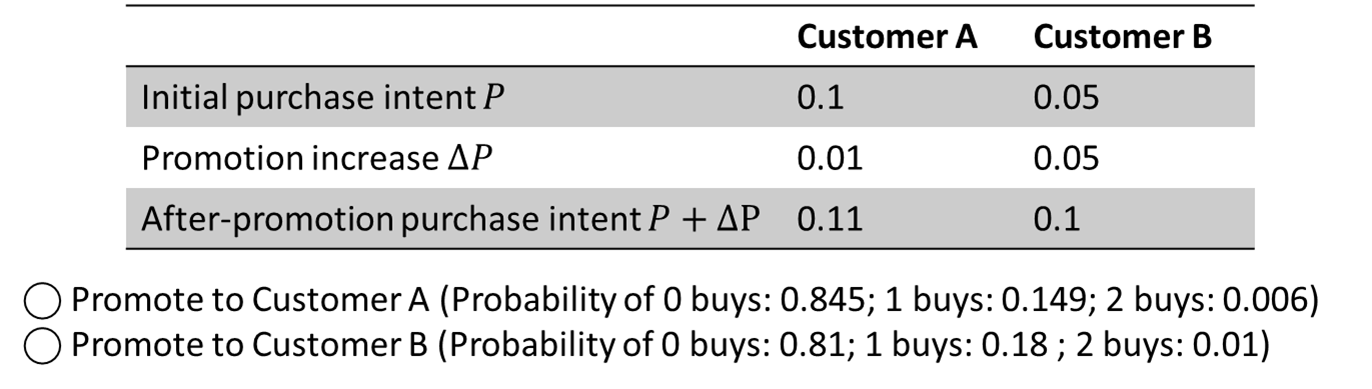}
    \caption{Example of problem from the treatment condition. The same problem from the control condition removes the purchase probabilities in the brackets.}
    \label{fig:experiment-problem-example}
\end{figure}

\paragraph{Experiment conditions}
Experiment conditions are designed to investigate whether providing the induced lottery PMFs change decision behavior.
Subjects are randomly assigned to an information condition \(g\in\{\mathrm{T},\mathrm{C}\}\). Treatment (T) subjects are provided the probability mass functions \((R_{A0}, R_{A1}, R_{A2})\) and \((R_{B0}, R_{B1}, R_{B2})\) at the time of choice, whereas control (C) subjects are not shown these PMFs. 
We also use a within-subject design to study sensitivity to payoff magnitudes.
Each participant experiences two magnitude conditions \(c\in\{\text{low}, \text{high}\}\) that differ only in the absolute magnitude of the investment, with condition \(\text{low}\) corresponding to the lower magnitude (\$30) and condition \(\text{high}\) corresponding to the higher magnitude (\$100).
We observe aggregate choices in each \((g,c)\) condition cell as counts of selecting \(A\) and \(B\), denoted \(N_{g,c,A}\) and \(N_{g,c,B}\), from which the empirical probability of choosing \(B\) (bRate) is calculated \(\widehat{p}_{g,c} = N_{g,c,B}/(N_{g,c,A}+N_{g,c,B})\).

\begin{table}[htbp]
\centering
\caption{Notation for the investment choice problem.}
\label{tab:notation-investment-problem}
\small
\begin{tabular}{ll}
\hline
Symbol & Meaning \\
\hline
\multicolumn{2}{l}{\textit{Problem Features}}\\
\(A,B\) & Two Bernoulli components, corresponding to Bernoulli variables $X_A$ and $X_B$\\
\(P_A^{\mathrm{ini}}, P_B^{\mathrm{ini}}\) & Initial success probabilities \\
\(\Delta P_A, \Delta P_B\) & Investment-induced probability increments \\
\(P_A^{\mathrm{aft}}, P_B^{\mathrm{aft}}\) & Post-investment success probabilities implied by the chosen action \\
% \(\mathrm{ANV}\) & Increment difference, sign(\(\Delta P_A - \Delta P_B\)) \\
\(R_{A0}, R_{A1}, R_{A2}\) & PMF of total successes \(0,1,2\) under option \(A\) \\
\(R_{B0}, R_{B1}, R_{B2}\) & PMF of total successes \(0,1,2\) under option \(B\) \\
\midrule
\multicolumn{2}{l}{\textit{Experiment Conditions and Responses}}\\
\(g\) & Information condition, \(g\in\{\mathrm{T},\mathrm{C}\}\) \\
\(c\) & Magnitude condition, \(c\in\{\text{low},\,\text{high}\}\) \\
\(N_{g,c,A}, N_{g,c,B}\) & Counts of choosing \(A\) or \(B\) in group \(g\), condition \(c\) \\
\(\widehat{p}_{g,c}\) & Target, empirical probability of choosing \(B\) (bRate), \(N_{g,c,B}/(N_{g,c,A}+N_{g,c,B})\) \\
\hline
\end{tabular}
\end{table}

\paragraph{Stimuli}
We generated a stimulus pool of 1,873 decision problems from a discrete grid
$\mathcal{P} = \{0,\,0.01,\,0.05,\,0.1,\,0.2,\,0.3,\,0.4,\,0.5,\,0.6,\,0.7,\,0.8,
\,0.9,\,0.95,\,0.99,\,1\}$. %of 1,873 unique 
For each component, $P^{\mathrm{ini}} \in \mathcal{P}$ and
$\Delta P \in \mathcal{P}\setminus\{0\}$ were chosen subject to
$P^{\mathrm{ini}} + \Delta P \le 1$.
Problems $(A, B)$ were retained only when the two investment increments were
comparable, i.e.\ $\Delta P_B \in [\Delta P_A,\;\Delta P_A + 0.10]$,
and the labels $A$/$B$ were swapped with probability 0.5 to remove position bias.
The problems are grouped into 132 batches with 15 problems each.
Each participant was assigned to exactly one batch.

\paragraph{Procedure}
The 15 problems were presented twice under different magnitude conditions, yielding 30 decisions per participant.
The payoff of one successful sale is 30 and 100 units for the low and high magnitude conditions, respectively.
Before answering, participants take a comprehension quiz to verify understanding of the task mechanics.
Incentives were performance-based: one trial was selected at random,
the outcome was resolved by sampling from the choice's induced PMF.
For each of the problems we record choice counts
$(N_{g,c,A},\,N_{g,c,B})$, from which $\widehat{p}_{g,c}$ is calculated.

\paragraph{Participants}
The experiment was conducted online via Credamo with a total of $N = 2640$ participants (1540 female; $M_\text{age}=27.4$), leading to 20 choices per problem. Six were excluded due to low answering time.
73.2\% reported no prior experience with risk decision experiments.
Median task completion time was 691\,s (IQR: 532--938\,s).
Participants were randomly assigned to an information condition:
$n_{\mathrm{T}} = 1314$ (treatment) and $n_{\mathrm{C}} = 1320$ (control).
Participants are compensated properly for their time, and are incentivised via a performance-based bonus.

\section{Behavioral Patterns}\label{sec:formulation-behavioral-finding}
In this section we explore the choice patterns under combinatorial risk.
We first examine how participants respond to key problem features by looking at the preference for the dominant option. Specifically, we ask how they choose when one option offers a higher probability increment, and when both options offer the same increment but one has a higher initial, and therefore post-investment, success probability.
Then, we investigate the effect of payoff magnitude and information treatment, respectively.

\subsection{Preference for the Dominant Option}
As shown in \Cref{sec:formulation-property}, the features of a combinatorial-risk problem determine both the expected value and the variance of its induced lotteries. We now examine whether participants favor the dominant option, either because it yields a higher expected value or because it yields lower variance when expected values are equal.

\begin{figure}[htbp]
    \centering
    \includegraphics[width=0.6\linewidth]{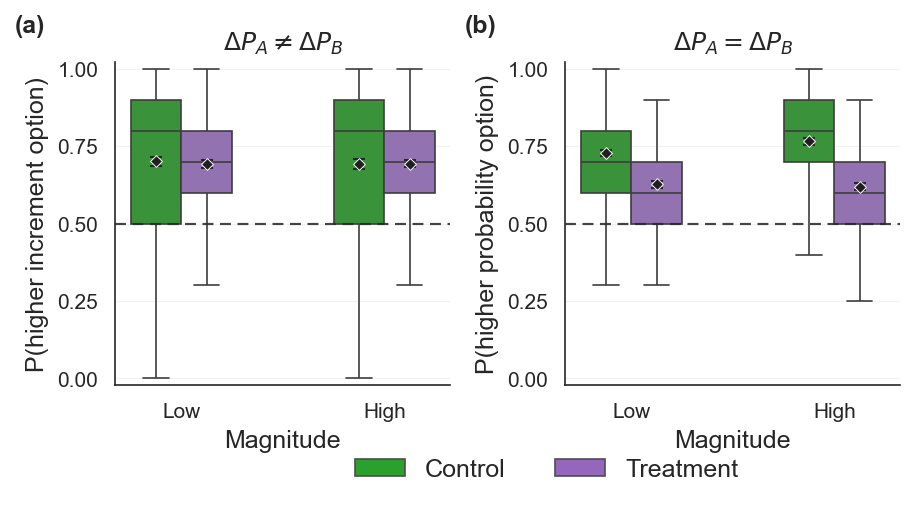}
    \caption{Boxplot of proportion of choosing the dominant option, with mean proportion marked along with 95\% confidence interval. (a) Proportion of choosing the option with higher probability increment in problems with $\Delta P_A\ne \Delta P_B$, and (b) proportion of choosing the higher initial, and thus after-investment, probability option in problems with $\Delta P_A=\Delta P_B$.}
    \label{fig:behavior-dominant}
\end{figure}

When one option $i^\ast\in\{A,B\}$ is dominant with a larger increment $\Delta P_{i^\ast}$, participants in both the treatment condition (low-magnitude: $t(1015)=32.65,\, p<0.001$, Cohen's $d=1.025$; high-magnitude: $t(1015)=31.99,\, p<0.001$, Cohen's $d=1.004$) and the control condition (low-magnitude: $t(1015)=25.66,\, p<0.001$, Cohen's $d=0.805$; high-magnitude: $t(1015)=23.29,\, p<0.001$, Cohen's $d=0.731$) favor the dominant option (see \Cref{fig:behavior-dominant} (a)).

When the two options have the same increment $\Delta P_A=\Delta P_B$, participants in both the treatment (low-magnitude: $t(856)=22.45,\, p<0.001$, Cohen's $d=0.767$; high-magnitude: $t(856)=21.27,\, p<0.001$, Cohen's $d=0.727$) and the control condition (low-magnitude: $t(856)=42.60,\, p<0.001$, Cohen's $d=1.456$; high-magnitude: $t(856)=51.435,\, p<0.001$, Cohen's $d=1.758
$) tend to choose the option with the higher initial, and therefore after-investment, success probability (see \Cref{fig:behavior-dominant} (b)).
This pattern is consistent with the \textit{certainty effect}: allocating the increment to the option with a higher initial success probability pushes the higher of the two success probabilities closer to one, which reduces outcome variance.

\subsection{Effect of Payoff Magnitude}\label{sec:formulation-behavior-magnitude}

\begin{figure}[htbp]
    \centering
    \includegraphics[width=0.6\linewidth]{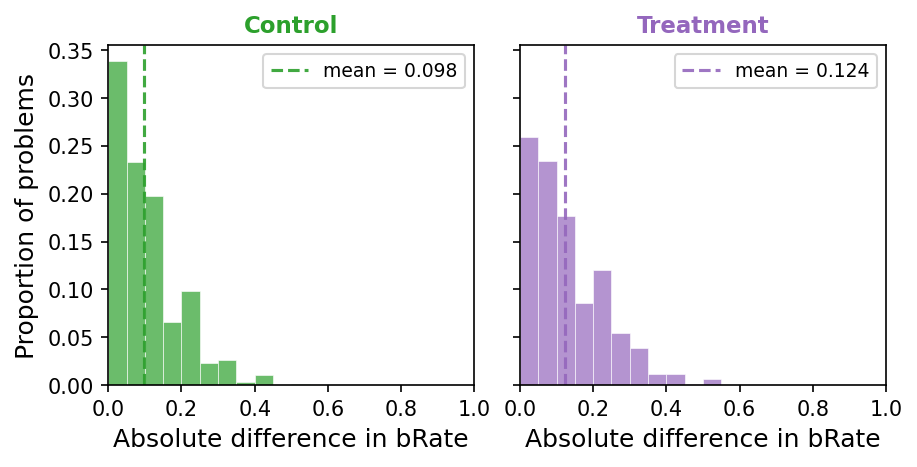}
    \caption{Absolute difference in choice probability between high- and low- magnitude problems.}
    \label{fig:behavior-magnitude}
\end{figure}

Payoff magnitude changes choice patterns in 66.2\% of the problems for the control group and 74.1\% for the treatment group, suggesting participants are sensitive to payoff magnitudes.
Furthermore, the effect of payoff magnitude is problem-dependent, and the dependency is affected by the information treatment.

For the control group participants, higher magnitude leads to higher sensitivity to combinatorial risk features: the change in bRate when the payoff magnitude increases is predicted by $P_B^\text{ini} - P_A^\text{ini}$ (Pearson's $r=0.1837$, $p<0.001$), $P_B^\text{aft} - P_A^\text{aft}$ (Pearson's $r=0.1802$, $p<0.001$).
In terms of the induced lottery, participants are more sensitive to the difference in variance $\text{Var}(S_B)-\text{Var}(S_A)$ when payoff magnitude is higher (Pearson's $r=-0.1836$, $p<0.001$).

Meanwhile, the treatment group exhibits different pattern of dependency on the payoff magnitude. None of the previously mentioned predictors are effective when the PMF is provided.
% (Pearson's $r=-0.0066$, $p=0.775$; Pearson's $r=-0.0051$, $p=0.826$; Pearson's $r=-0.0080$, $p=0.730$).
In fact, the change in bRate as payoff magnitude increases are uncorrelated for the treatment and control group (Pearson $r = -0.0012$, $p = 0.960$), suggesting systematically different decision-making strategies given the induced lotteries explicitly.
% The difference in proportion of choosing $B$ between low- and high-magnitude problems centers around 0 in both the treatment condition (paired $t$-test: $t(1872)=0.0789,\,p=0.8946$, Cohen's $d=0.0095$) and the control condition (paired $t$-test: $t(194)=0.5081,\, p=0.6119$, Cohen's $d=0.0365$).

\subsection{Effect of Information Treatment}
\Cref{fig:behavior-dominant} shows that participants in different information treatment conditions exhibit different responses to differences in options, as measured by $\Delta P_B-\Delta P_A$ and $P_B^{\text{ini}}-P_A^{\text{ini}}$.
When one of the options has higher probability increment, the treatment and control groups have similar preferences for the dominant option (low-magnitude: $t(1015)=-1.129,\,p=0.2592$; high-magnitud: $t(1015)=0.058,\,p=0.9540$), while participants in the treatment condition exhibit smaller variance in choice probabilities across problems (low-magnitude: $F(1,2030)=84.421,\, p<0.001$; high-magnitude: $F(1,2030)=104.400,\, p<0.001$).
When the two options have the same increment, participants in the control condition has significantly stronger preference for the higher initial-probability option than the treatment participants (low-magnitude: $t(857)=13.035,\,p<0.001$, Cohen's $d=0.445$; high-magnitud: $t(857)=19.156,\,p<0.001$, Cohen's $d=0.654$).

Notably, providing the induced lottery PMF does not make it easier to achieve higher payoff.
Although their choice patterns differ significantly, both groups achieve a similar amount of successes, with treatment group achieving slightly lower
(
low-magnitude: $t(1873)=-1.0403,\,p=0.2983$, Cohen's $d=-0.024$;
high-magnitude: $t(1873) = -0.0639,\, p=0.949$, Cohen's $d =-0.002$
).
\section{Benchmark Models}\label{sec:benchmarks}
\subsection{Heuristic Models}

Let the decision maker choose between investing in \(A\) and investing in \(B\). If \(A\) is selected, the investment changes only the success probability of \(A\), yielding post-investment probabilities \(P_A^{\mathrm{aft}} = P_A^{\mathrm{ini}} + \Delta P_A\). If \(B\) is selected, then \(P_B^{\mathrm{aft}} = P_B^{\mathrm{ini}} + \Delta P_B\). Under either action, the induced outcome distribution for the total number of successes \(U \in \{0,1,2\}\) is the convolution of two independent Bernoulli trials with parameters \((p_1,p_2)\). In particular, \((p_1,p_2) = (P_A^{\text{aft}},P_B^{\text{ini}})\) if A is chosen, and \((p_1,p_2) = (P_A^{\text{ini}},P_B^{\text{aft}})\) otherwise. For any pair \((p_1,p_2)\) we have
\[
\Pr(U=0) = (1-p_1)(1-p_2),\qquad
\Pr(U=1) = p_1(1-p_2) + (1-p_1)p_2,\qquad
\Pr(U=2) = p_1 p_2.
\]
We write \(\mathbf{R}^A = (R_{A0},R_{A1},R_{A2})\) for the distribution implied by investing in \(A\) and \(\mathbf{R}^B = (R_{B0},R_{B1},R_{B2})\) for the distribution implied by investing in \(B\), which correspond exactly to the PMF features provided to subjects in the treatment condition.
We consider several benchmark theories for how the decision maker maps the available information into a choice propensity.

\paragraph{Probability Gain Model}
A parsimonious baseline is a \emph{probability gain} model, in which the decision maker compares the immediate improvements \(\Delta P_A\) and \(\Delta P_B\) and invests in the option with the larger increment. This mechanism predicts that the choice probability
\begin{equation}
    \widehat{p}_{g,c} = \sigma\left((\Delta P_B-\Delta P_A)/T\right),
\label{eq:formulation-baseline-probability-gain}
\end{equation}
where $\sigma(\cdot)$ is the logistic function, $T$ is the temperature parameter.

\begin{remark}
        The probability gain model (\ref{eq:formulation-baseline-probability-gain}) with $T\to0$ maximizes expected value.
\end{remark}
% \begin{proof}
%     The expected number of success $\mathbb E[S] = p_1 + p_2$. Therefore, 
%     \[
%     \mathbb E[S\vert \text{choose }A] = p_A^{\text{ini}} + p_B^{\text{ini}} + \Delta P_A,\qquad
%     \mathbb E[S\vert \text{choose }B] = p_A^{\text{ini}} + p_B^{\text{ini}} + \Delta P_B
%     \]
% Therefore choosing the option with higher probability gain maximizes the expected value.
% \end{proof}

\paragraph{After-investment Probability Model}
A second class of models evaluates the probability of success after investment, treating the investment as a way to increase a single Bernoulli parameter. In this view, the decision maker compares \(P_A^{\mathrm{aft}}\) and \(P_B^{\mathrm{aft}}\) and chooses according to
\begin{equation}
    \widehat{p}_{g,c} = \sigma\left((P_B^{\mathrm{aft}}-P_A^{\mathrm{aft}})/T\right),
\label{eq:formulation-baseline-after-investment}
\end{equation}

\paragraph{Tail-Probability Models}
A natural alternative to expected-value related models is that decision makers attend to tail events of the total success count \(U\), particularly when the full outcome distribution is explicitly presented. The tail-probability model posits that choices are driven by a weighted tradeoff between the best-case probability \(\Pr(U=2)\) and the worst-case probability \(\Pr(U=0)\) under each investment.
\[
U_{\mathrm{tail}}(A)=\omega_0 (1-\Pr(U=0\mid A))+\omega_2 \Pr(U=2\mid A),\]\[
U_{\mathrm{tail}}(B)=\omega_0 (1-\Pr(U=0\mid B)) +\omega_2 \Pr(U=2\mid B),
\]
where \(\omega_2,\omega_0\in [0,1]\) capture the relative emphasis on the upside and downside tails. The choice propensity is then modeled by a logistic comparison of tail scores. For the model $\text{Tail}(\omega_0, \omega_2)$
\begin{equation}
    \widehat{p}_{g,c}=\sigma\!\left(\bigl(U_{\mathrm{tail}}(B)-U_{\mathrm{tail}}(A)\bigr)/T\right),
\end{equation}
where the score difference can be written directly in terms of the combinatorial-risk features
\[
U_{\mathrm{tail}}(B)-U_{\mathrm{tail}}(A)
=\omega_2\Bigl(P_A^{\mathrm{ini}}\Delta P_B - P_B^{\mathrm{ini}}\Delta P_A\Bigr)
-\omega_0\Bigl((1-P_A^{\mathrm{ini}})\Delta P_B - (1-P_B^{\mathrm{ini}})\Delta P_A\Bigr).
\]
% This form highlights how tail-based evaluation induces interactions between baseline probabilities and probability increments: the same increment can have markedly different effects depending on the current levels of \(P_A^{\mathrm{ini}}\) and \(P_B^{\mathrm{ini}}\), potentially favoring more balanced allocations when the upside tail is salient or prioritizing failure avoidance when the downside tail is salient.

\subsection{State-Space Utility Models}
\label{sec:state-space-benchmarks}

The preceding heuristic models define choice rules directly on salient features of the
combinatorial-risk problem, such as probability increments and after-investment success
probabilities. We next introduce a class of theory-grounded benchmarks that evaluate the post-investment probability state induced by each action, and characterize the risk attitude for combinatorial risk. 
Here we present the resulting decision models, while formal derivation is provided in Appendix~\ref{sec:risk-attitude}.

Let \(m=(p_A,p_B)\in[0,1]^2\)
denote a probability state. Investing in \(A\) or \(B\) leads respectively to
\[
m^A=(P_A^{\mathrm{ini}}+\Delta P_A,\; P_B^{\mathrm{ini}}),
\qquad
m^B=(P_A^{\mathrm{ini}},\; P_B^{\mathrm{ini}}+\Delta P_B).
\]
A state-space utility model assigns a value \(u(m)\) to each probability state and predicts choices
by comparing \(u(m^A)\) and \(u(m^B)\)
\begin{equation}
\widehat{p}_{g,c} =
    \sigma\left(
    ({u(m^B)-u(m^A)})/{T}
    \right),
\end{equation}
where \(T>0\) is a temperature parameter. 
The state-space utility directly evaluate the probability state instead of the induced PMF over total successes.
This leads to a multi-dimensional generalization of \citet{pratt1964risk} utility function. With a directional field $d(m)$ for the risk premium, we derive a risk-attitude index
\[
\mathcal{A}(m) = - \dfrac{H_u(m)}{\nabla u(m)^\top d(m)},
\]
where $H_u(m)$ is the Hessian matrix.
With this risk-attitude index, and assuming $d(m)=(1,1)$ to be indifferent between the dimensions, we proceed to define risk neutrality, constant risk averse, and decreasing risk averse models.

\paragraph{Risk Neutral Model}
Risk neutrality on the state space $\mathcal{A}(m)=0$ leads to the linear utility
\[
u_{\mathrm{RN}}(m)=a+\beta^\top m.
\]
In this case, the initial probabilities cancel out:
\(
u_{\mathrm{RN}}(m^B)-u_{\mathrm{RN}}(m^A)
=
\beta_2\Delta P_B-\beta_1\Delta P_A.
\)
Risk-neutral state-space utility therefore reduces to a
weighted comparison of probability increments.

\paragraph{Constant Risk Averse Model}
The constant-risk-aversion assumes $\mathcal{A}(m) = M$ is a constant. Under rank-one assumption $M = \frac{bb^\top}{b_1+b_2}$, where \(b=(b_1,b_2)\), we derive the decision model explicitly
\[
u_{\mathrm{CRA}}(m)
=
A\exp(-b^\top m)+\eta^\top m+D,
\]
for constants $A,D\in \mathbb R$ and $\eta \in \mathbb R^2$ with $\eta^\top(1,1)=0$.
Let \(
z^{\mathrm{ini}}=b_1P_A^{\mathrm{ini}}+b_2P_B^{\mathrm{ini}},
\)
we have
\[
u_{\mathrm{CRA}}(m^B)-u_{\mathrm{CRA}}(m^A)
=
A e^{-z^{\mathrm{ini}}}
\left(
e^{-b_2\Delta P_B}-e^{-b_1\Delta P_A}
\right)
-\eta_1(\Delta P_A+\Delta P_B).
\]
% This benchmark introduces interactions between initial probabilities and increments, because
% the marginal value of an increment depends on the current state \(z^{\mathrm{ini}}\).

\paragraph{Decreasing Risk Averse Model}
Finally, we consider decreasing-risk-aversion utility under the same rank-one assumption, where $\mathcal{A}(m) =\rho(b^\top m)\,\frac{bb^\top}{b_1+b_2},
$ for some decreasing function $\rho(\cdot)$. The utility can be written as
\[
u_{\mathrm{DRA}}(m)
=
F(b^\top m)+\eta^\top m+D,
\qquad
\eta^\top(1,1)=0,
\]
where the curvature of \(F\) determines how risk aversion varies \(-\frac{F''(z)}{F'(z)}=\rho(z)
\). This gives
\[
u_{\mathrm{DRA}}(m^B)-u_{\mathrm{DRA}}(m^A)
=
F(z^{\mathrm{ini}}+b_2\Delta P_B)
-
F(z^{\mathrm{ini}}+b_1\Delta P_A)
-\eta_1(\Delta P_A+\Delta P_B).
\]
We implement two parametric versions. The linear decreasing-risk-aversion model with
\[
\rho_{\mathrm{lin}}(z)=\rho_0-\rho_1 z,\qquad \rho_1\ge 0,
\]
and the exponential decreasing-risk-aversion model with
\[
\rho_{\mathrm{exp}}(z)=\rho_0 e^{-\gamma z},
\qquad
\rho_0,\gamma\ge 0.
\]
These models offer theory-grounded alternatives to simple heuristics. They characterize risk preferences over the state space of success probabilities. The resulting decision models exhibit interactions between initial success probabilities and probability increments. 
Unlike prospect-theoretic models, they do not require evaluation of the induced PMF, making them more behaviorally plausible.

\subsection{Risky Choice Models on Induced Lotteries}\label{sec:benchmarks-risky}
We also consider classical risky choice models applied on the induced lotteries.
Each model below maps a lottery \(\mathbf{R}\) to a scalar valuation \(V(\mathbf{R})\), and the choice propensity is a logistic comparison
\begin{equation}
  \widehat{p}_{g,c}=\sigma\!\left(\bigl(V(\mathbf{R}^B)-V(\mathbf{R}^A)\bigr)/T\right),
\label{eq:formulation-risky-choice-generic}
\end{equation}
where \(\sigma(\cdot)\) is the logistic function and \(T\) the temperature parameter.
Throughout we use the power value function \citep{tversky1992advances}
\[
  v(x;\alpha)=x^{\alpha},\qquad \alpha>0,
\]
and the Log Odds Linear weighting function \citep{gonzalez1999shape}
\[
\pi(p;\alpha_w,\beta_w)=\frac{\beta_w\,p^{\alpha_w}}{\beta_w\,p^{\alpha_w}+(1-p)^{\alpha_w}},
\qquad \alpha_w>0,\ \beta_w>0.
\]

\paragraph{Expected Utility (EU)}
The expected-utility model \citep{von1944theory} values each lottery by the probability-weighted sum of outcome utilities,
\begin{equation}
  V_{\mathrm{EU}}(\mathbf{R})=\sum_{u=0}^{2} R_u\, v(x_u;\alpha)
  = R_1 + R_2\,2^{\alpha}.
\label{eq:formulation-eu}
\end{equation}
The single curvature parameter \(\alpha\) encodes the agent's risk attitude over the induced success count, and the
choice propensity follows \eqref{eq:formulation-risky-choice-generic}.
\begin{remark}
  EU with \(\alpha=1\) reduces to the probability gain model \eqref{eq:formulation-baseline-probability-gain}.
\end{remark}

\paragraph{Prospect Theory (PT)}
Prospect theory \citep{Kahneman1979prospect} replaces objective probabilities with subjective decision weights by applying the weighting function
\begin{equation}
  V_{\mathrm{PT}}(\mathbf{R})=\sum_{u=0}^{2}\pi(R_u;\alpha_w,\beta_w)\, v(x_u;\alpha)
  = \pi(R_1;\alpha_w,\beta_w) + \pi(R_2;\alpha_w,\beta_w)\,2^{\alpha}.
\label{eq:formulation-pt}
\end{equation}
This separable transformation accommodates the overweighting of small probabilities and underweighting of large ones.
The choice propensity again follows \eqref{eq:formulation-risky-choice-generic}.

\paragraph{Cumulative Prospect Theory (CPT)}
Cumulative prospect theory \citep{tversky1992advances} resolves the violations of stochastic dominance that can arise
under separable weighting by applying the weighting function to cumulative probabilities.
Ordering the outcomes \(x_0<x_1<x_2\) in the gain domain, the rank-dependent decision weights are,
\[
\pi_2=\pi(R_2),\qquad
\pi_1=\pi(R_1+R_2)-\pi(R_2),\qquad
\pi_0=1-\pi(R_1+R_2),
\]
where \(\pi(\cdot)=\pi(\cdot;\alpha_w,\beta_w)\) and \(\pi_0+\pi_1+\pi_2=1\) by construction. The lottery valuation is
\begin{equation}
  V_{\mathrm{CPT}}(\mathbf{R})=\sum_{u=0}^{2}\pi_u\, v(x_u;\alpha)
  = \bigl(\pi(R_1+R_2)-\pi(R_2)\bigr) + \pi(R_2)\,2^{\alpha},
\label{eq:formulation-cpt}
\end{equation}
and the choice propensity follows \eqref{eq:formulation-risky-choice-generic}. 
% As all outcomes are non-negative, only the gain-domain weighting branch is active, so the loss-aversion and loss-curvature parameters of the full \citet{tversky1992advances} specification are inert.

\subsection{Benchmark Evaluation}

\begin{table}[htbp]
\centering
\caption{Model test performances on the control and treatment conditions.}
\label{tab:theory_performance}
\scriptsize
\begin{tabular}{rcccccc}
\toprule
\multirow{2}{*}{Model} & \multicolumn{3}{c}{Control} & \multicolumn{3}{c}{Treatment} \\
\cmidrule(r){2-4}\cmidrule(r){5-7}
 & CE$_\mathrm{test}$ $\downarrow$ & MSE$_\mathrm{test}$ $\downarrow$ & Acc$_\mathrm{test}$ $\uparrow$ & CE$_\mathrm{test}$ $\downarrow$ & MSE$_\mathrm{test}$ $\downarrow$ & Acc$_\mathrm{test}$ $\uparrow$ \\
\midrule
\multicolumn{7}{l}{\textit{Combinatorial risk - heuristic models}} \\
Delta-diff                    & 0.6411 & 0.0713 & 0.6707 & 0.6439 & 0.0405 & 0.7547 \\
Aft-prob                      & 0.6196 & 0.0590 & 0.8813 & 0.6703 & 0.0529 & 0.7227 \\
Tail(1,0)                     & 0.5948 & 0.0460 & \textbf{0.8920} & 0.6456 & 0.0405 & 0.8040 \\
Tail(0,1)                     & 0.6932 & 0.0965 & 0.2853 & 0.6932 & 0.0644 & 0.4440 \\
Tail($\omega_0$,$\omega_2$)   & 0.5844 & 0.0427 & 0.8760 & 0.6259 & \underline{0.0316} & \textbf{0.8240} \\
\midrule
\multicolumn{7}{l}{\textit{Combinatorial risk - state space utility models}} \\
Risk Neutral                          & 0.6413 & 0.0713 & 0.6707 & 0.6436 & 0.0403 & 0.7547 \\
Const. Risk Averse                   & 0.6365 & 0.0695 & 0.6707 & 0.6424 & 0.0399 & 0.7547 \\
Lin. Decr. Risk Averse       & 0.6160 & 0.0597 & 0.7520 & 0.6362 & 0.0367 & 0.7787 \\
Exp. Decr. Risk Averse  & 0.6366 & 0.0695 & 0.6707 & 0.6420 & 0.0397 & 0.7547 \\
\midrule
\multicolumn{7}{l}{\textit{Induced lottery - risky choice models}} \\
EU                            & 0.5843 & 0.0426 & 0.8760 & 0.6258 & \underline{0.0316} & \underline{0.8213} \\
PT                            & \textbf{0.5773} & \textbf{0.0401} & \underline{0.8893} & \underline{0.6256} & \underline{0.0316} & \underline{0.8213} \\
CPT                           & \underline{0.5781} & \underline{0.0412} & 0.8867 & \textbf{0.6252} & \textbf{0.0315} & \textbf{0.8240} \\
\bottomrule
\end{tabular}
\end{table}

We partition the dataset into a training and a test set using an 80/20 split.
We fit free parameters on the training set by minimizing MSE between predicted and empirical choice propensities $\widehat p_{g,c}$ across problems. We then report cross-entropy (CE), mean-squared error (MSE) and accuracy (Acc.) on the test set, separately for control and treatment conditions (see \Cref{tab:theory_performance}).

Simple heuristics such as the probability-gain model and the after-investment probability model are insufficient. Tail-based models improve prediction, with downside emphasis $\text{Tail}(1,0)$ outperforming upside $\text{Tail}(0,1)$, and the flexible $\text{Tail}(\omega_0,\omega_2)$ performing best (with $\omega_0=0.74,\, \omega_2=0.30$) among the heuristics, consistent with stronger sensitivity to failure risk than to best-case outcomes.

State-space utility models capture part of this behavior but not all of it. Relative to risk neutrality, allowing for risk aversion improves fit, especially when risk aversion decreases with the initial state.
This suggests that subjects' behavior is shaped not just by aversion to risk per se, but by how that aversion varies with the initial success probability.

Prospect-theoretic models on the induced lotteries provide the best overall fits.
In treatment, CPT performs the best.
In control, PT performs the best, while CPT remains close.
Overall, the results indicate that choices depend on both initial success probabilities and increments, exhibit tail-risk sensitivity, and PT/CPT are competitive benchmarks for subsequent analyses.

However, the overall prediction performances of the existing theories are still dissatisfactory. Best-performing prospect-theoretic models are not behaviorally viable especially for the control group, where the calculation of the PMF is intractable for humans.
\section{Model Discovery via Symbolic Regression}\label{sec:experiments}
\subsection{Hybrid Symbolic Regression with Ontology-Guided Exploration}\label{sec:method}

\begin{figure}[htbp]
    \centering
    \includegraphics[width=0.5\linewidth]{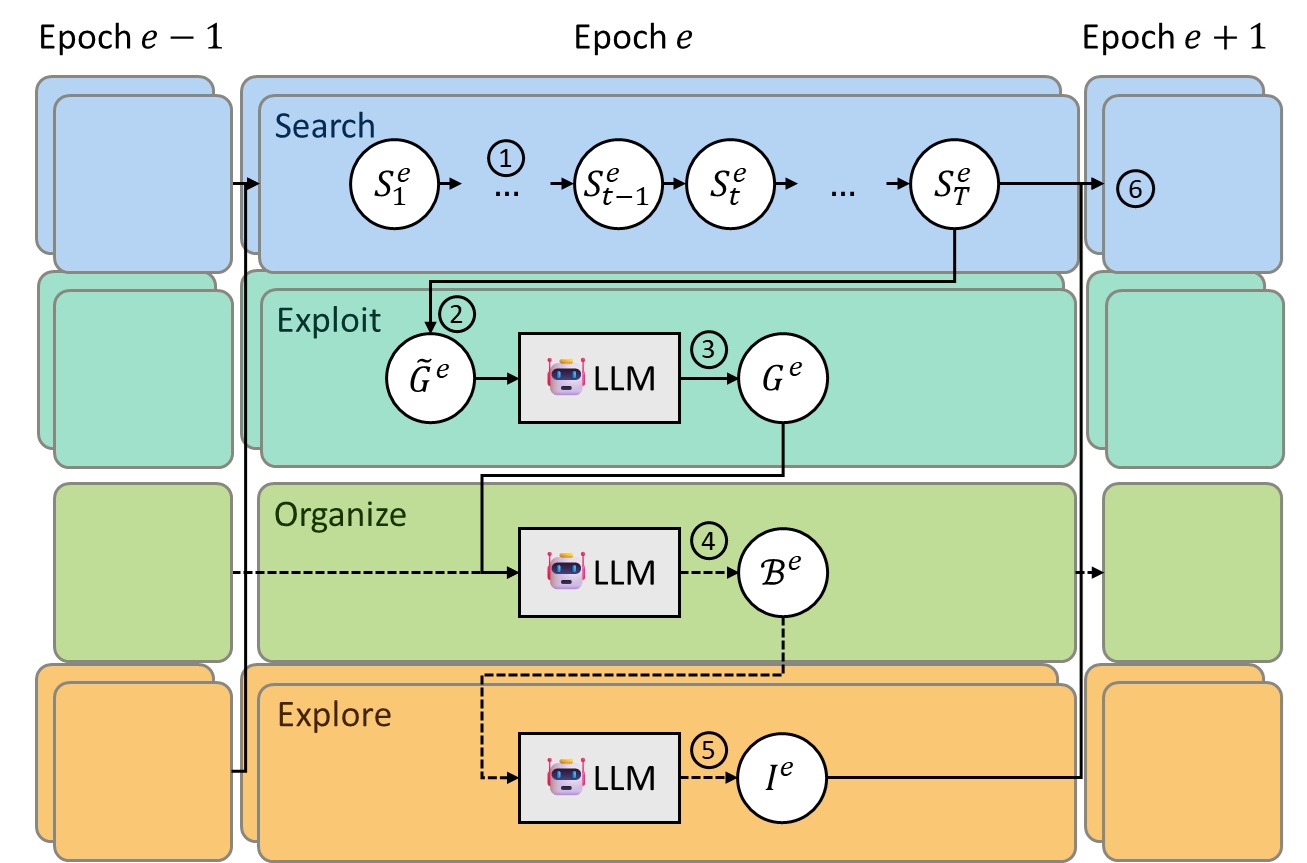}
    \caption{Overview of the proposed method. Each dataset is treated as an island. In every epoch, $\textcircled{1}$ each island independently searches for $T$ generations, $\textcircled{2}$ the Pareto-optimal candidates are extracted as elites and $\textcircled{3}$ passed to an LLM for local refinement (Exploit). $\textcircled{4}$ All Pareto models are parsed into a shared ontology and periodically organized by an LLM into theory-grounded concepts and categories. $\textcircled{5}$ A per-island Explorer queries the ontology for underexplored regions and calls an LLM to synthesize new candidate expressions. Finally, $\textcircled{6}$ synthesized expressions are cross-routed to all schema-compatible islands, enabling knowledge transfer across conditions.}
    \label{fig:method-overview}
\end{figure}

The proposed framework combines evolutionary symbolic regression with LLM-based generation in an epoch-based island architecture (\Cref{fig:method-overview}). 
Expressions are represented as binary trees whose nodes correspond to operators, variables, or constants. Each node has a type (Scalar or Vector) to allow for vector-valued variables such as probabilities and outcomes. Unlike general-purpose symbolic regression methods that treat inputs as flat feature vectors, this typed representation retains the original input structure throughout the search process.

Each epoch begins with a {search} phase, in which NSGA-II with constant optimization is run for $T$ generations to improve predictive fit and expression simplicity. Separate populations of candidate expressions evolve in parallel across distinct experimental conditions (“islands”), allowing the search to adapt to condition-specific structure.
Pareto-optimal candidates are then passed to an {exploit} phase, where an LLM performs targeted local revisions intended to refine promising models while preserving interpretability.

At the end of each epoch, Pareto-frontier models are {parsed} into a shared ontology and cross-evaluated on all islands. The ontology is then {organized} through clustering and LLM-based analysis, which names recurring functional forms, identifies higher-level behavioral concepts, and groups related models into theory-relevant categories. Guided by this representation, a per-island {explorer} queries the ontology for promising but underexplored concepts and prompts the LLM to generate new candidate expressions. These candidates are then cross-routed to all compatible islands, enabling discoveries in one condition to seed exploration in others. This creates an iterative search–exploit–organize–explore loop that aims not only to improve predictive performance, but also to accumulate interpretable and reusable knowledge about the structure of decision behavior.
Details of the methodology are presented in the Appendix \Cref{sec:append-method}. 

A validation result is presented in the Appendix \Cref{sec:append-experiment-13k}. Running the framework on the {\tt Choices13k} dataset \citep{peterson2021using} successfully re-discovered classical theories such as prospect theory. 
Ablation study confirms that introduced Exploitation and Exploration steps effectively improve the performance.
It also confirms that utilizing structured input facilitates model discovery by reducing the dimension of the search space.
%This process rigorously enforces feature-space constraints: unrestricted migration is permitted symmetrically between identical environments and asymmetrically from sparse Control to feature-rich Treatment datasets. Conversely, models migrating from Treatment to Control are strictly filtered to exclude unavailable probability features. 

% Schema compatibility governs which transfers are valid: a model may only be routed to islands whose feature sets are a superset of the features the model references. This permits unrestricted transfer between conditions sharing identical schemas, while preventing models that reference unavailable features from entering incompatible populations.

\subsection{Discovered Descriptive models}\label{sec:experiment-discovered-models}

\begin{figure}[htbp]
    \centering
    \includegraphics[width=0.4\linewidth]{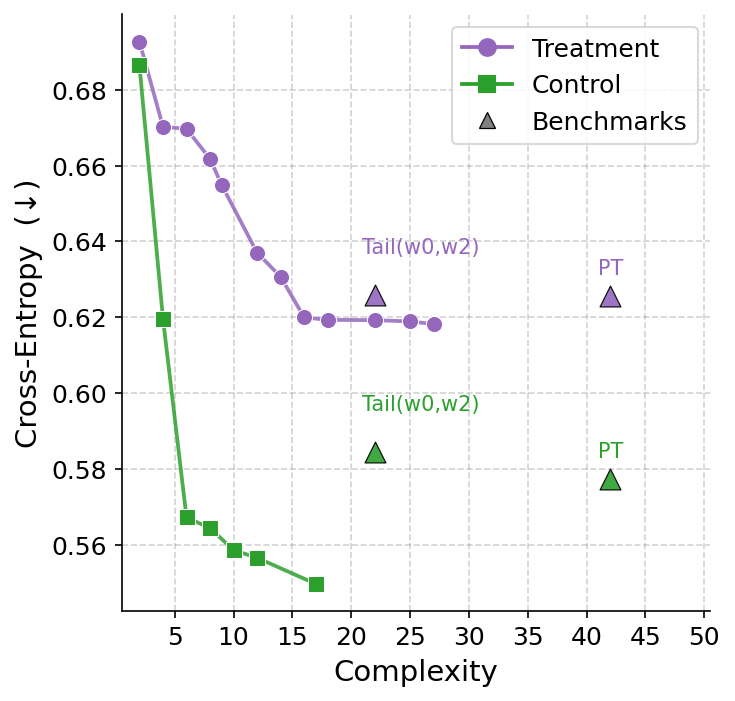}
    \caption{Pareto frontier of models on the {\tt Combinatorial Risk} testset data.}
    \label{fig:exp-comb-pareto}
\end{figure}

Symbolic regression successfully identified a rich set of descriptive models (see \Cref{fig:exp-comb-pareto}).
The SR results confirm that some ingredients of the benchmark models are genuinely important.
At the same time, the symbolic models show that these ingredients can be assembled into hybrid models that make better predictions than those hand-crafted ones.

In both conditions, the difference in after-investment success probability, $P_A^{\mathrm{aft}}-P_B^{\mathrm{aft}}$, repeatedly appears in Pareto optimal models, including as a low-complexity model on its own.
This is consistent with the exploratory finding that when probability increments are equal, participants tend to prefer the option with the higher after-investment success probability.
Likewise, probability increment $\Delta P_A-\Delta P_B$ also appears in many of the Pareto-optimal models.
In that sense, SR validates the behavioral findings in \Cref{sec:formulation-behavioral-finding}.

Meanwhile, SR also discovers new descriptive models that involve instance-dependent transformation.
The Pareto frontier models discovered by SR dominate the benchmark models, achieving better predictive performance with simpler expressions (see \Cref{fig:exp-comb-pareto}).
The SR frontier shows that much of the behavior can be captured without committing to full prospect-theoretic evaluation of the induced lottery.
In the following part of this section, we present and analyze the discovered models on the Pareto frontiers for both conditions.

\begin{table}[ht]
\centering
\caption{Pareto-optimal models for the control condition of the combinatorial-risk experiment. CE$_\mathrm{test}$ = cross-entropy; MSE$_\mathrm{test}$ = mean squared error; Acc$_\mathrm{test}$ = accuracy (test set).}
\label{tab:pareto_notable_combRisk_C}
\scriptsize
\setlength{\tabcolsep}{4pt}
\begin{tabular}{rclll p{6cm}}
\toprule
Model & Complexity & CE$_\mathrm{test}\,\downarrow$ & MSE$_\mathrm{test}\,\downarrow$ & Acc$_\mathrm{test}\,\uparrow$ & Expression \\
\midrule
Ini-A   & 2 & 0.6865 & 0.0930 & 0.4653 & $P_A^{\mathrm{ini}}$ \\
Aft-prob   & 4 & 0.6196 & 0.0590 & 0.8680 & $P_A^{\mathrm{aft}} - P_B^{\mathrm{aft}}$ \\
Ini-scaled Aft-prob   & 6 & {0.5672}$^\dagger$ & 0.0364 & 0.8680 & $\mathrm{sgp}(P_A^{\mathrm{aft}} - P_B^{\mathrm{aft}}, P_A^{\mathrm{ini}})$ \\
Ini-scaled Aft-prob   & 8 & {0.5643}$^\dagger$ & 0.0353 & 0.8680 & $\mathrm{sgp}(P_A^{\mathrm{aft}} - P_B^{\mathrm{aft}}, \mathrm{sgp}(P_A^{\mathrm{ini}}, c_{1}))$ \\
Hybrid   & 10 & {0.5585}$^\dagger$ & 0.0330 & \textbf{0.8867} & $\Delta P_A - \Delta P_B + \mathrm{sgp}(P_A^{\mathrm{aft}} - P_B^{\mathrm{aft}}, P_A^{\mathrm{ini}})$ \\
Hybrid   & 12 & \underline{0.5565}$^\dagger$ & \underline{0.0323} & \underline{0.8840} & $\Delta P_A - \Delta P_B + \mathrm{sgp}(P_A^{\mathrm{aft}} - P_B^{\mathrm{aft}}, \mathrm{sgp}(P_A^{\mathrm{ini}}, c_{1}))$ \\
Hybrid  & 17 & \textbf{0.5496}$^\dagger$ & \textbf{0.0297} & 0.8680 & $c_{1} \cdot (\mathrm{dot}(\mathbf{o}_A, \mathbf{R}_A)- \mathrm{dot}(\mathbf{o}_B, \mathbf{R}_B))$
  \newline\hspace*{1em} $- c_{2} \cdot \mathrm{sgp}(-P_A^{\mathrm{aft}} + P_B^{\mathrm{aft}}, c_{3})$ \\
\bottomrule
\end{tabular}
\vspace{2pt}
\parbox{0.95\linewidth}{
\footnotesize
$^\dagger$ CE$_\mathrm{test}$ is lower than the baseline CE.
}
\end{table}

For the control condition (see \Cref{tab:pareto_notable_combRisk_C}), Pareto-optimal models do not rely on the explicit PMF except the most complex hybrid model.
This implies that human decisions are not sensitive to exact expected value for the decision under combinatorial risk.
After-investment difference $P_A^{\mathrm{aft}} - P_B^{\mathrm{aft}}$ appears both as a simple model (Complexity = 4) and as a component in many of the models, transformed with a sensitivity parameter.
Interestingly, SR discovered models (Complexity = 6 \& 8) with instance-dependent transformation. 
They utilize the after-investment difference $P_A^{\mathrm{aft}} - P_B^{\mathrm{aft}}$ as the core quantity that determines the preference direction, and use initial success probability $P_A^\textrm{ini}$ for power transformation.
The model with lower CE (Complexity = 10 \& 12) further introduces difference in probability increments $\Delta P_A - \Delta P_B$.
The most complex hybrid model (Complexity = 17) introduces the expected payoff $\mathrm{dot}(\mathbf{o}_A, \mathbf{R}_A)- \mathrm{dot}(\mathbf{o}_B, \mathbf{R}_B)$, which takes both expected number of success and the payoff magnitude into account.
However, it does not capture the interaction between the magnitude effect and the after-investment probability observed in \Cref{sec:formulation-behavior-magnitude}.

\begin{table}[ht]
\centering
\caption{Pareto-optimal models for the treatment condition of the combinatorial-risk experiment. CE$_\mathrm{test}$ = cross-entropy; MSE$_\mathrm{test}$ = mean squared error; Acc$_\mathrm{test}$ = accuracy (test set).}
\label{tab:pareto_notable_combRisk_T}
\scriptsize
\setlength{\tabcolsep}{4pt}
\begin{tabular}{rclll p{7cm}}
\toprule
Model & Complexity & CE$_\mathrm{test}\,\downarrow$ & MSE$_\mathrm{test}\,\downarrow$ & Acc$_\mathrm{test}\,\uparrow$ & Expression \\
\midrule
 Constant  & 2 & 0.6928 & 0.0642 & 0.4600 & $c_{1}$ \\
Aft-prob  & 4 & 0.6703 & 0.0529 & 0.7267 & $P_A^{\mathrm{aft}} - P_B^{\mathrm{aft}}$ \\
Aft-prob    & 6 & 0.6697 & 0.0527 & 0.7000 & $P_A^{\mathrm{aft}} - P_B^{\mathrm{aft}} + c_{1}$ \\
Hybrid   & 8 & 0.6618 & 0.0488 & 0.7600 & $P_A^{\mathrm{aft}}- P_B^{\mathrm{aft}} + \Delta P_A - \Delta P_B$ \\
Log-EV   & 9 & 0.6549 & 0.0433 & 0.6987 & $\ln(\mathrm{dot}(\mathbf{R}_A, \mathbf{o}_A)/\mathrm{dot}(\mathbf{R}_B, \mathbf{o}_B))$ \\
Power-EV   & 12 & {0.6370} & 0.0366 & 0.7093 & $P_A^{\mathrm{ini}} - \mathrm{sgp}(-\mathrm{dot}(\mathbf{o}_A, \mathbf{R}_A) + \mathrm{dot}(\mathbf{o}_B, \mathbf{R}_B), c_{1})$ \\
Power-EV   & 14 & {0.6307} & 0.0341 & 0.7773 & $P_A^{\mathrm{ini}} - \mathrm{sgp}(\mathrm{dot}(\mathbf{o}_B, \mathbf{R}_B) - \mathrm{dot}(\mathbf{o}_A, \mathbf{R}_A), c_{1}) - c_{2}$ \\
Hybrid   & 16 & \beatbase{0.6199} & 0.0291 & \textbf{0.8267} & $P_A^{\mathrm{aft}} - P_B^{\mathrm{aft}} - c_{1} \cdot \mathrm{sgp}(-\mathrm{dot}(\mathbf{o}_A, \mathbf{R}_A) + \mathrm{dot}(\mathbf{o}_B, \mathbf{R}_B), c_{2})$ \\
Hybrid   & 18 & \beatbase{0.6193} & 0.0288 & 0.8133 & $P_A^{\mathrm{aft}} - P_B^{\mathrm{aft}} - c_{1} \cdot \mathrm{sgp}(-\mathrm{dot}(\mathbf{o}_A, \mathbf{R}_A) + \mathrm{dot}(\mathbf{o}_B, \mathbf{R}_B), c_{2}) + c_{3}$ \\
Hybrid   & 22 & \beatbase{0.6192} & 0.0288 & 0.8120 & $P_A^{\mathrm{aft}} - P_B^{\mathrm{aft}} + \Delta P_A - \Delta P_B $\newline\hspace*{1em} $- c_{1} \cdot \mathrm{sgp}(-\mathrm{dot}(\mathbf{o}_A, \mathbf{R}_A) + \mathrm{dot}(\mathbf{o}_B, \mathbf{R}_B), c_{2}) + c_{3}$ \\
Hybrid   & 25 & \underline{\beatbase{0.6189}} & \underline{0.0287} & \underline{0.8213} & $P_A^{\mathrm{aft}} - P_B^{\mathrm{aft}} - c_{1} \cdot \mathrm{sgp}(-\mathrm{dot}(\mathbf{o}_A, \mathbf{R}_A) + \mathrm{dot}(\mathbf{o}_B, \mathbf{R}_B), c_{2})$\newline\hspace*{1em} $+ \ln(\mathrm{dot}(\mathbf{o}_A, \mathbf{R}_A)/\mathrm{dot}(\mathbf{o}_B, \mathbf{R}_B))$ \\
Hybrid   & 27 & \textbf{\beatbase{0.6182}} & \textbf{0.0283} & 0.8160 & $P_A^{\mathrm{aft}} - P_B^{\mathrm{aft}} - c_{1} \cdot \mathrm{sgp}(-\mathrm{dot}(\mathbf{o}_A, \mathbf{R}_A) + \mathrm{dot}(\mathbf{o}_B, \mathbf{R}_B), c_{2})$\newline\hspace*{1em} $+ \ln(\mathrm{dot}(\mathbf{o}_A, \mathbf{R}_A)/\mathrm{dot}(\mathbf{o}_B, \mathbf{R}_B)) + c_{3}$ \\
\bottomrule
\end{tabular}

\vspace{2pt}
\parbox{0.95\linewidth}{
\footnotesize
$^\dagger$ CE$_\mathrm{test}$ is lower than the baseline CE.
}

\end{table}

\begin{figure}[htbp]
    \centering
    \includegraphics[width=0.4\linewidth]{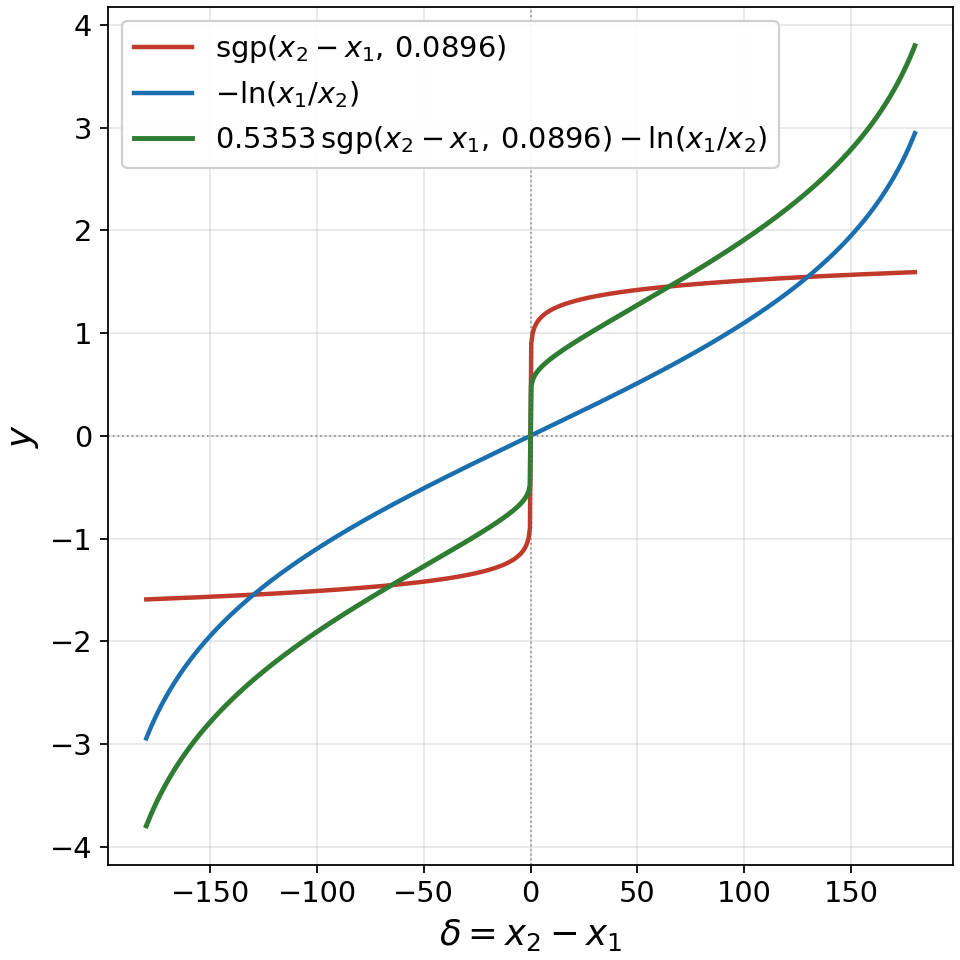}
    \caption{Expected payoff difference transformations discovered from the treatment group data.}
    \label{fig:exp-comb-function-payoff}
\end{figure}

For the treatment condition (See \Cref{tab:pareto_notable_combRisk_T}), the most predictive models combine combinatorial risk features with the induced lottery features to make predictions.
Models with high predictive performance (Complexity $\ge$ 12) all involve the expected payoff term $\mathrm{dot}(\mathbf{o}_A, \mathbf{R}_A)- \mathrm{dot}(\mathbf{o}_B, \mathbf{R}_B)$, suggesting attention is attracted from the combinatorial risk features to the induced lottery features (PMFs).
The expected payoff appear in two forms: $\log$-transformation (e.g., Complexity = 9) and power transformation (e.g., Complexity = 12 \& 14).
These nonlinear transformations are visualized in \Cref{fig:exp-comb-function-payoff}, with numerical optimization performed on the combinatorial risk training set data.
The most predictive models (Complexity = 25 \& 27) utilize both forms, suggesting a non-trivial relation between choice propensity and the expected payoff of the induced lottery.
Notice that the $\log$-transformation is invariant to the change in payoff magnitude.
The magnitude effect manifests through the power-transformed expected payoff difference, and the functional form discovered suggests a diminishing sensitivity to magnitude increase.
Meanwhile, the after-investment difference $P_A^{\mathrm{aft}} - P_B^{\mathrm{aft}}$ term remains in most of the Pareto-optimal models.

%%%%%%%%%%%%%%%%%%%%%%%%

\section{Residual Analysis for the Treatment Effect}\label{sec:residual-analysis}

Providing the induced lottery PMF changes behavior. 
\Cref{sec:experiment-discovered-models} shows that discovered models for the treatment group exhibit higher reliance on the expected payoff term, suggesting that treatment participants use the additional distributional information to make choices. 
To further investigate the treatment effect, we analyze the residuals of the best-performing control model without the expected-payoff term (Complexity = 12) when it is transferred to the treatment data.
Let \(\widehat P_C(x)\) denote the prediction of the control model. 
For treatment problem \(i\), we define the residual
\[
r_i = y_i - \widehat P_C(x_i),
\]
where \(y_i\) is the observed probability of choosing \(B\). 
Because \(\widehat P_C\) uses only combinatorial risk features 
\((P^{\mathrm{ini}}, P^{\mathrm{del}}, P^{\mathrm{aft}})\), any systematic structure in \(r_i\) reflects variation in treatment choices that is not explained by the behavioral rule learned from the control condition. 
In particular, correlations between \(r_i\) and PMF-derived quantities identify how the displayed distributional information shifts choices relative to the control model.
% \begin{figure}[htbp]
%     \centering
%     \includegraphics[width=0.4\linewidth]{figs/exp-residual-analysis.png}
%     \caption{Spearman correlations between induced PMF features and residuals from the Control model transferred to the Treatment condition.}
%     \label{fig:exp-residual-analysis}
% \end{figure}
The residuals show a selective relationship with the PMF features. % (see \Cref{fig:exp-residual-analysis}). 
They are negatively correlated with \(R_{A2}\) 
\((\rho=-0.221,\,p<0.001)\) and \(R_{B1}\) 
\((\rho=-0.099,\,p=0.006)\), while positively correlated with \(R_{B2}\) 
\((\rho=0.125,\,p<0.001)\) and \(R_{A1}\) 
\((\rho=0.115,\,p=0.002)\). 
Intuitively, participants in the treatment group prefer option $B$ more if investing in B leads to higher probability of two successes. To gain a deeper insight beyond correlations, we fit residual models to explain the behavioral differences.

\paragraph{Residual Models}
The residual models take as input only the PMF shown to treatment participants, including three outcome levels
$(o_{j0},\,o_{j1},\,o_{j2})$ and their associated probabilities
$(R_{j0},\,R_{j1},\,R_{j2})$ for $j\in\{A,B\}$, together with the control model's prediction $\hat{P}_C(x)$.
Each option's prospect-theory value is
$V(j)=\sum_k w(R_{jk};\gamma)\,v(o_{jk};\alpha_v)$, where the value function and the probability weighting function follows the form in \Cref{sec:benchmarks-risky}.
The choice probability is then
\[
\hat{P}(B\mid x)=\sigma\!\bigl(\alpha\,\mathrm{logit}\,
\hat{P}_C(x)+\beta_0+\beta_1\,\Delta V\bigr)
\]
with $\Delta V=V(B)-V(A)$, and all five parameters
$(\alpha,\beta_0,\beta_1,\alpha_v,\gamma)$ are fit jointly by
maximum likelihood on the training set from the treatment group.
We compare two variants that differ only in what the value function sees: {C+PT\,(w/)} applies $v(\cdot)$ to outcomes in their monetary units (e.g.\
$\$0,\$30,\$60$), so the fitted $\alpha_v$ must simultaneously
capture diminishing sensitivity and the outcome scale, whereas
{C+PT\,(w/o)} first normalizes outcomes by the payoff magnitude
to $\{0,1,2\}$.

\paragraph{Evaluation Method}
We compare the augmentation of residual models to the original SR-discovered models (C \& T) and a PT-only benchmark.
We use the same set of evaluation metrics (CE, MSE, and Acc) for comparison.
To quantify the uncertainty of the model evaluation, we employ bootstrapping by generating $B = 1000$ bootstrap replicates of the test set by resampling with replacement. For each bootstrap sample $b \in \{1, \dots, B\}$, we recompute the evaluation metrics for each model variant. The standard error ($\mathrm{SE}$) for each metric is then estimated as the empirical standard deviation of the bootstrap distribution
\begin{equation}
\mathrm{SE}(\theta) = \sqrt{\frac{1}{B-1} \sum_{b=1}^{B} \left(\hat{\theta}^*_b - \bar{\theta}^*\right)^2}
\end{equation}
where $\hat{\theta}^*_b$ represents the metric value computed on the $b$-th bootstrap sample, and $\bar{\theta}^*$ is the sample mean of the bootstrap estimates \(\bar{\theta}^* = \frac{1}{B} \sum_{b=1}^{B} \hat{\theta}^*_b\).

\paragraph{Results}

\begin{figure}[htbp]
    \centering
    \includegraphics[width=1\linewidth]{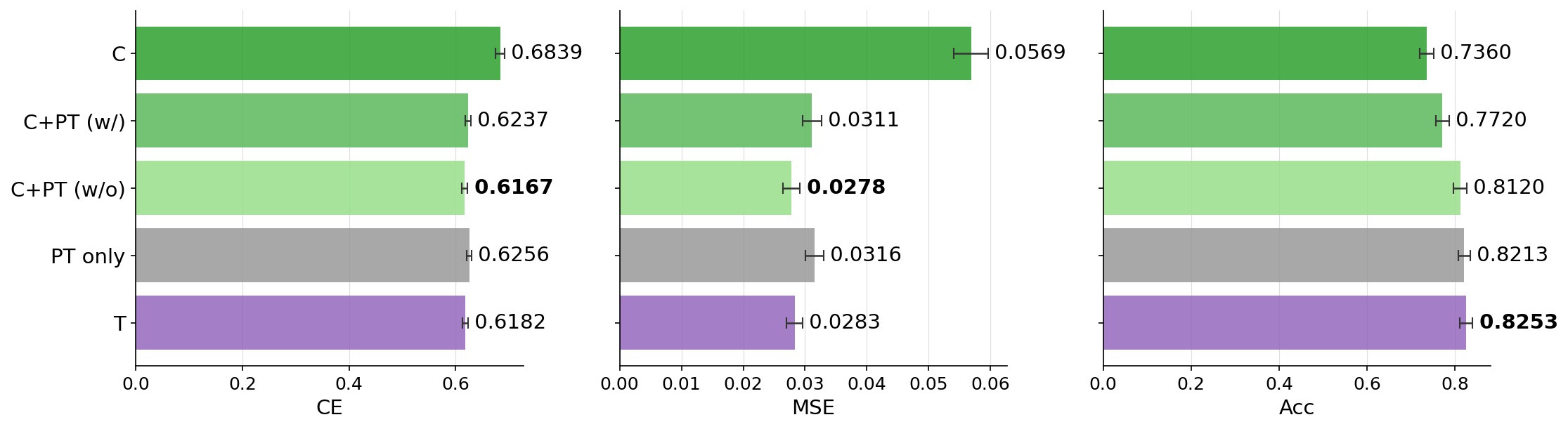}
    \caption{Comparison of control model, residual model augmentation, PT-only benchmark, and treatment model on the treatment group data, with bootstrapped error bars ($\pm$ 1 SE).}
    \label{fig:exp-residual-model}
\end{figure}

Examining the performance metrics shows that PT residual models effectively augment control model (see \Cref{fig:exp-residual-model}), suggesting behavioral differences can be explained by the incorporation of PMFs for decision making.
The original control model transfers poorly to the treatment data, with CE \(=0.6839\) and MSE \(=0.0569\). 
Augmentation control model with a PT residual model boosts performance to be comparable to the best-performing treament model (T).
Notably, the PT residual model without payoff magnitude (C+PT (w/o)) achieves the best CE \(=0.6167\) and MSE \(=0.0278\), significantly outperfroming the control model (C).
It also beats the prospect theory benchmark (PT only).
This suggests the behavioral shift when the induced lottery is revealed can be effectively explained by incorporating prospect-theoretic evaluation of induced lotteries with the control model that utilizes combinatorial risk features.
\section{Discussion}\label{sec:discussion}

This paper studies decision making under \emph{combinatorial risk} that manifests in many real-world decision problems.
The challenge for descriptive modeling is that exact evaluation of the induced lottery can be cognitively demanding.
In our investment-allocation paradigm, subjects did not behave as if they were explicitly optimizing the full induced PMF.
Instead, choices were shaped by a set of psychologically plausible quantities, most notably the after-investment success probabilities and the difference in probability increments.
In particular, subjects tended to favor the dominant option when one option offered a larger probability increment, and when increments were equal they favored the option with the higher initial/after-investment success probability.
An important finding is that providing more distributional information substantially changes behavior, but does not necessarily help decision makers act more advantageously.
Participants in the treatment condition were less responsive to the combinatorial-risk features than those in control, and exhibited compressed variation in choice propensities.

The discovered descriptive models sharpen the behavioral interpretation. 
First, for the control condition, the discovered models consistently revolve around the difference in after-investment success probability, sometimes modulated by baseline probability.
Notably, exact expected payoff appears only in the most complex control models and with weak influence, suggesting that subjects are not tracking expected payoff directly.
These simple models also make quantitatively better predictions of the decisions than hand-crafted benchmark models. 
Second, for the treatment condition, symbolic regression identifies useful models that combine combinatorial-risk features with nonlinear transformations of expected payoff, suggesting attention is attracted from the combinatorial risk features to the induced lottery features.
Residual analysis confirms the hypothesis, as augmentation with a PT residual model effectively predicts treatment group behavior.

Methodologically, the paper also contributes a framework for descriptive model discovery. The ontology-guided hybrid search improves search efficiency and the final model quality by organizing reusable concepts, functional forms, and categories, allowing the knowledge to accumulate over time.
More broadly, this work illustrates where the SR framework becomes scientifically useful. 
SR not only serves as a flexible search procedure that can reproduce human-discovered theories, but also inherits theoretically meaningful structure from prior theories and reorganizes into new composites.
This facilitates theory discovery in new experiment settings, where standard theories provide useful benchmarks but do not fully explain how people adapt their decisions to reward structure and computational complexity.

\subsection*{Limitations and Future work}
Several limitations should be noted.
First, although the current task aims to take a step from lottery problems to more realistic settings, it is still simple with two Bernoulli components and a single indivisible investment.
This simplicity is a strength for isolating combinatorial-risk mechanisms, but broader generalization remains to be established.
Second, the data are aggregate choice proportions rather than individual-level repeated decisions, so the discovered models describe average behavior and may mask heterogeneity in strategy use. 
Third, current experiments are constrained by computational cost, so only a limited number of expressions have been explored by the symbolic regression. The resulting Pareto frontier enjoys simplicity, but it is reasonable to expect more predictive models are to be discovered.
For example, a model with higher complexity might explain the observed interaction between initial success probability and the magnitude effect.
Finally, although symbolic regression yields compact and interpretable rules, it does not establish process-level validity; some discovered expressions may be excellent approximations to behavior without corresponding exactly to the internal computations subjects perform.

These limitations point naturally to future work. It would be valuable to test richer combinatorial-risk settings with more than two components, unequal rewards, or repeated decisions.
Besides, the framework could be extended to individual-level symbolic models or mixture formulations that capture heterogeneous strategies.
Interestingly, the present results suggest that combinatorial risk is a productive domain for joint progress in behavioral theory and interpretable machine learning: it is rich enough to generate nontrivial behavioral patterns, yet structured enough for symbolic model discovery to recover reusable and psychologically meaningful regularities.
Future work may also seek to establish descriptive models with better prediction power and process-level validity, aiming for a deeper understanding of human strategies under combinatorial risk.

\section{Conclusion}\label{sec:conclusion}
This paper introduces combinatorial-risk and studies decision-making via an investment-allocation task. 
Experiment shows that people responded systematically to combinatorial risk features, and revealing the PMFs alters behavior.
To model these patterns, we employ an ontology-guided symbolic regression framework that pushes beyond hand-crafted benchmarks.
The discovered models achieve better predictions, and suggest that human decisions under combinatorial risk are guided by salient problem features rather than the evaluation of the induced distribution.

\bibliographystyle{apalike}
\bibliography{sample}

\newpage
\begin{appendices}
\section{Method}\label{sec:append-method}

\begin{algorithm}[htbp]
\caption{Hybrid Symbolic Regression with Ontology-guided Exploration}
\label{alg:method}
\small
\begin{algorithmic}[1]
\Require Initial Population $S^0$, Search Rounds $T$, Max Epochs $E$
\Ensure Final Ontology $\mathcal{B}^E$

\State $S_{0}^0 \gets S^0$; \quad $\mathcal{B}^0 \gets \emptyset$
\For{$e = 1$ to $E$}
    % \State \Comment{\textbf{Phase 1: Search}}
    \For{$t = 1$ to $T$}
        \State $S_t^e \gets \textsc{Search}(S_{t-1}^e)$
    \EndFor
    % \State \Comment{\textbf{Phase 2: Exploit}}
    \State $\widetilde{G}^e \gets \textsc{ParetoElites}(S_T^e)$
    \State $G^e \gets \textsc{Exploit}(\widetilde{G}^e)$; \quad $S_T^e \gets S_T^e \cup G^e$
    % \State \Comment{\textbf{Phase 3: Parse \& Organize}}
    \State $\mathcal{B}^e \gets \textsc{Organize}(\textsc{ParetoElites}(S_T^e),\; \mathcal{B}^{e-1})$
    % \State \Comment{\textbf{Phase 4: Explore \& Cross-Route}}
    \State $I^e \gets \textsc{Explore}(\mathcal{B}^e)$
    \State Inject $I^e$ into origin island; cross-route compatible models to other islands
    \State $S_0^{e+1} \gets S_T^e$
\EndFor

\State \Return $\mathcal{B}^E$
\end{algorithmic}
\end{algorithm}

\paragraph{Expression Tree Representation}
%The operator alphabet includes binary operators $\{+,\,-,\,\times,\,\div,\,\mathtt{dot},\,\mathtt{signed\_pow}\}$ and optionally unary operators $\{\exp,\log,|\cdot|,\mathrm{neg},\mathrm{sum}\}$.

We employ a tree-based representation for symbolic decision models.
Each candidate is represented as a typed expression tree whose internal nodes are operators and whose leaf nodes are input features or constants. 
A central design choice is the treatment of vector-valued features.
Each gamble is naturally described by a discrete probability distribution with vectors of outcomes and probabilities. Normatively grounded decision models operate directly on these distributions as atomic objects. 
Flattening these vectors into individual scalar features $x_1, x_2, \ldots, x_K$,
as is standard in general-purpose symbolic regression, discards the data structure and leads to significantly larger search space and less meaningful expressions.
To accommodate this structure, each node is annotated with a type, \textsc{scalar} or \textsc{vector}, and tree generation is performed with explicit shape constraints propagated from parent to child. 
Shape compatibility is enforced by construction during subtree generation, crossover, and mutation.
This design allows the search to natively express vector-level operations such as dot product $\mathtt{dot}(\mathbf{p},\, u(\mathbf{x}))$, while guaranteeing that every candidate expression produces a scalar prediction at its root.

The features are specified by a \emph{feature schema}, a dataset-specific declaration that enumerates the named
features of each alternative and their types (scalar or vector).%, and constraints (e.g., that the probability vector sums to one) for expression simplification.
For example, the schema for the {\tt choices13k} dataset includes two vector features per alternative:
\texttt{outcomes} and \texttt{probs}.

\paragraph{\texttt{Search}}

Model search is cast as a bi-objective optimization problem. The first objective is
predictive fit, measured by the log-likelihood under a softmax decision rule (maximized).
The second objective is expression complexity, measured by the number of nodes in the expression tree (minimized). 
The two objectives are optimized jointly using the Non-dominated Sorting Genetic Algorithm~II (NSGA-II) \citep{deb2002fast}.
In contrast to standard GP approaches that combine fit and complexity into a single penalized objective \citep{cranmer2023interpretable}, Pareto-based search avoids the need to specify \emph{a priori} the exchange rate between these competing criteria \citep{kommenda2015complexity}. This is especially desirable here because expression complexity is not merely a regularizer, but also a scientifically meaningful criterion tied to interpretability and theoretical simplicity. 
At each generation, offspring are produced via tournament selection followed by subtree crossover and mutation. Tournament selection draws a random subset of individuals and returns the winner according to a lexicographic criterion: lowest non-domination rank first (rank 0 being the Pareto-optimal front), then largest crowding distance, an estimate of local solution density in the objective space that rewards diversity.
Subtree crossover selects a random subtree in each of two parents and swaps them, subject to the constraint that the swapped subtrees share the same output shape. 
Each selected offspring is then subjected to exactly one mutation operator, chosen uniformly at random: \emph{subtree mutation} replaces a randomly selected child subtree with a freshly generated tree of matching output shape; \emph{node mutation} replaces an operator with a structurally compatible alternative (preserving input and output shapes); and \emph{constant perturbation} adds zero-mean Gaussian noise to all learnable scalar constants in the tree.
Finally, the constants of the returned model are fine-tuned via L-BFGS-B to maximize log-likelihood with the expression structure held fixed.

\paragraph{\texttt{Exploit}}
The Exploit step refines expressions locally to improve both predictive accuracy and interpretability.
At each epoch, the top-$K$ elites are selected from the Pareto frontier by log-likelihood.
The LLM is instructed to improve these $K$ expressions.
The system prompt specifies the background, the available variables, the set of operators, and task instructions (see \Cref{fig:method-prompt-system}).
The user prompt presents the $K$ elite expressions, each annotated with its complexity and log-likelihood, together with the diagnostic examples (see \Cref{fig:method-prompt-user-improve}).
To provide the LLM with context information, the prompt includes the average absolute error across all
elites. The three cases with the highest average error are also provided as diagnostic examples.
The LLM response is parsed into the expression tree representation. If parsing fails, a repair prompt is sent to the LLM requesting syntax correction.
If a converted model contains free numeric constants, those constants are re-optimized on the training data.
Finally, the proposed models are injected into the GP
population, where they compete alongside GP-generated candidates in subsequent generations.

\begin{figure}[htbp]
\centering
    \begin{tcolorbox}[
      % colback=white,
      % colframe=black,
      % boxrule=0.6pt,
      arc=2mm,
      left=2mm,
      right=2mm,
      top=2mm,
      bottom=2mm,
      width=0.75\linewidth
    ]
    \scriptsize
    \noindent You are an expert in behavioral decision models.
    You will help propose and/or improve mathematical expressions used to predict human choices between risky gambles.
    Each problem is a choice between two alternatives \(A\) and \(B\), and the target is the probability of choosing \(B\).
    The softmax decision rule is automatically applied to the output of the expression.
    Your task is to propose parsimonious expressions that effectively capture human choices under risk.
    
    \medskip
    \noindent \textbf{Available features per gamble (use EXACTLY these variable names):}\\
    \(\texttt{A\_outcomes}, \texttt{B\_outcomes}\): vector feature \texttt{outcomes} for alternatives \(A\) and \(B\).\\
    \(\texttt{A\_probs}, \texttt{B\_probs}\): vector feature \texttt{probs} for alternatives \(A\) and \(B\).
    
    \medskip
    \noindent \textbf{Operators available:}\\
    \(\texttt{+}\) \((S,S)\to S\) \(\mid\) \((V,V)\to V\): element-wise or scalar addition.\\
    \(\texttt{-}\) \((S,S)\to S\) \(\mid\) \((V,V)\to V\): element-wise or scalar subtraction.\\
    \(\texttt{*}\) \((S,S)\to S\) \(\mid\) \((S,V)\to V\) \(\mid\) \((V,S)\to V\): scalar-broadcast or scalar multiplication.\\
    \(\texttt{/}\) \((S,S)\to S\) \(\mid\) \((S,V)\to V\) \(\mid\) \((V,S)\to V\): scalar-broadcast or scalar division.\\
    \(\texttt{dot}\) \((V,V)\to S\): dot product (e.g., \(\texttt{dot(A\_outcomes, A\_probs)}\)).\\
    \(\texttt{signed\_pow}\) \((S,S)\to S\) \(\mid\) \((S,V)\to V\) \(\mid\) \((V,S)\to V\): sign-preserving power, \(\mathrm{sign}(x)\lvert x\rvert^{y}\).
    
    \medskip
    \noindent \textbf{Rules:}\\
      - Output only the logit score expression (the softmax rule is applied separately).\\
      - Keep expressions concise --- prefer lower complexity.\\
      - Use valid {\sc Python} or {\sc SymPy}-compatible syntax.\\
      - Each expression must be parseable by sympy.\\
      - Use ONLY the variable names listed above.
    \end{tcolorbox}
    \caption{System prompt for EXPLOIT with {\tt choices13k} schema.}    \label{fig:method-prompt-system}
\end{figure}

\begin{figure}[htbp]
\centering
    \begin{tcolorbox}[
      % colback=white,
      % colframe=black,
      % boxrule=0.6pt,
      arc=2mm,
      left=2mm,
      right=2mm,
      top=2mm,
      bottom=2mm,
      width=0.75\linewidth
    ]
    \scriptsize
    \noindent\textbf{Current Best Expressions}\\
    \{expression\_list\}
    
    \medskip
    \noindent\textbf{Hardest Prediction Cases}\\
    These are the cases where the current best expressions cannot predict well.\\
    \{error\_examples\}
    
    \medskip
    \noindent\textbf{Your Task}\\
    For EACH of the \{n\} expressions above, output exactly ONE improved variant in the SAME ORDER.\\
    Output exactly \{n\} numbered lines:\\
    \{output\_template\}
    
    \medskip
    \noindent\textbf{Important:}\\
      - Preserve the numbering (1., 2., \dots).\\
      - Each line must contain only the expression, no explanation.\\
      - If you cannot improve an expression, output it unchanged.
    \end{tcolorbox}
    \caption{User prompt template for EXPLOIT.}\label{fig:method-prompt-user-improve}
\end{figure}

\paragraph{\texttt{Organize}: An Ontology of Decision Models}

\begin{figure}[htbp]
    \centering
    \begin{subfigure}[b]{0.58\linewidth}
        \centering
        \includegraphics[width=\linewidth]{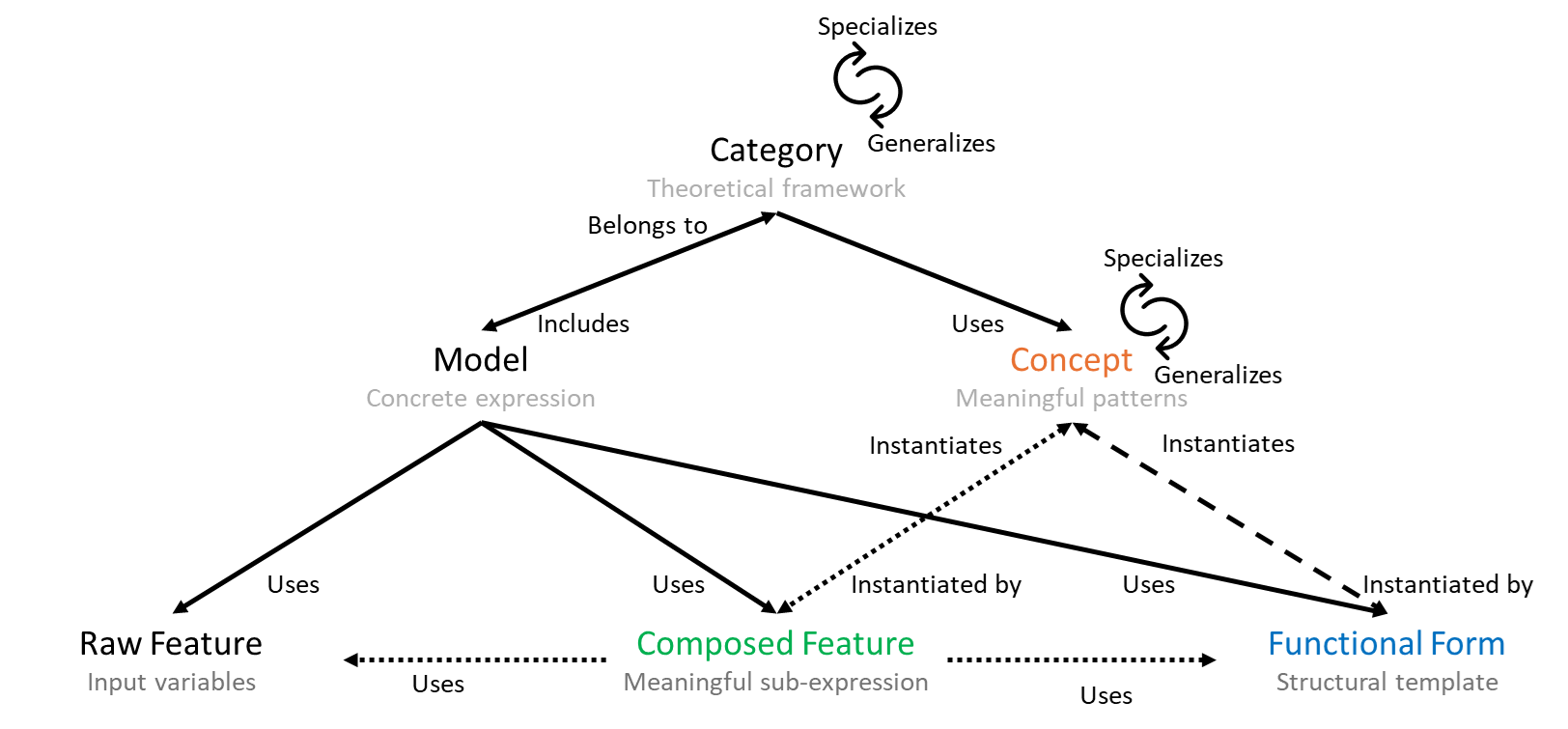}
        \caption{Node and edge types in the ontology graph.}
        \label{fig:method-ontology-graph}
    \end{subfigure}
    \hfill
    \begin{subfigure}[b]{0.4\linewidth}
        \centering
        \includegraphics[width=\linewidth]{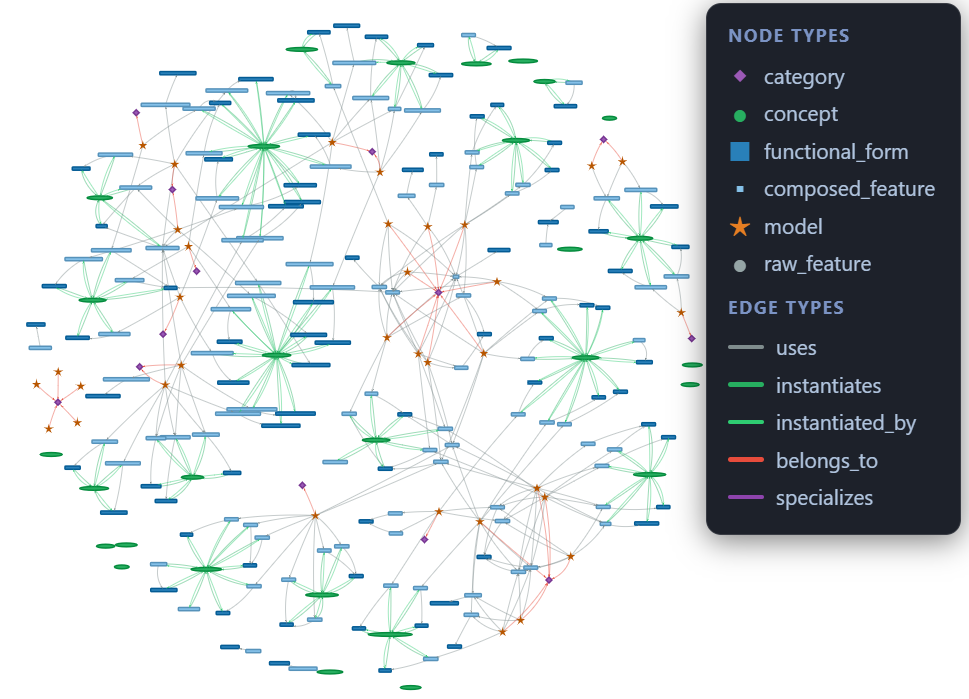}
        \caption{An example ontology graph.}
        \label{fig:method-ontology-example}
    \end{subfigure}
    \caption{Ontology graph schema and an illustrative example.}
    \label{fig:method-ontology}
\end{figure}

We maintain a graph-structured ontology that organizes discovered symbolic models and exposes reusable functional forms and semantic insights.
As illustrated in \Cref{fig:method-ontology-graph}, the ontology
contains six types of nodes: \textbf{Category}, \textbf{Concept},
\textbf{Functional Form}, \textbf{Composed Feature}, \textbf{Model}, and
\textbf{Raw Feature} (atomic inputs). Intuitively, \emph{categories} represent
broad theoretical frameworks (e.g., Expected Utility Theory); \emph{concepts}
represent reusable behavioral primitives (e.g., probability weighting);
\emph{functional forms} are abstract parameterized templates written in
placeholder notation (e.g., $v_1^{c_1}$, where $v_1$ is a feature vector and
$c_1$ is a tunable constant). A \emph{composed feature} is a concrete
sub-expression obtained by binding a functional form's slots to specific raw
features (e.g., $A\_\textit{outcomes}^{c_1}$), while a \emph{model} is a
complete expression.
%%%%%%%%%%%%%%%%%%%%%%%%%%%%
Edges encode typed relations. \texttt{SPECIALIZES}/\texttt{GENERALIZES}
capture hierarchical refinement within the same layer (Category-Category or
Concept-Concept. E.g., linear utility as a special case of power utility with $c_1{=}1$).
\texttt{INSTANTIATES}/\texttt{INSTANTIATED\_BY} ground abstract patterns in
interpretable constructs: composed features and functional forms instantiate
concepts. \texttt{USES} links each model to the raw features, composed
features, and functional forms it depends on; \texttt{BELONGS\_TO} assigns
models to one or more categories. These typed edges facilitates graph query.

Model nodes are clustered to form categories. Each model is represented as a set of composed features, and clustering follows the DP Mean algorithm \citep{kulis2012revisiting}: each new model is assigned to the most similar cluster in terms of Jaccard similarity, unless the similarity falls below some threshold. In that case, 
a new cluster is intialized to accomodate the model.

To support exploration control across epochs, non-model nodes carry an
exploration status (\texttt{hypothetical} $\rightarrow$ \texttt{targeted}
$\rightarrow$ \texttt{evidenced}), visit and failure counters, and an
intrinsic value tracking the best log-likelihood observed among linked models.
When the LLM proposes a new expression, a rules-based parser resolves any
composed-feature references by substituting their canonical sub-expressions,
extracts the feature bindings of each composed feature (slot $\mapsto$ raw
feature name), and synthesizes a live model with optimized constants. New
composed-feature and functional-form nodes are added as needed, deduplicated
by canonical expression string, and the resulting model is registered as a
\texttt{ModelNode} linked via \texttt{USES} and \texttt{BELONGS\_TO} edges.
The ontology thus accumulates a structured ``map'' of explored territory:
what has been tried, what is empirically supported, and which concept
neighborhoods remain underexplored.

\paragraph{\texttt{Explore}}
While \texttt{Exploit} refines existing candidates, it cannot propose expressions that are structurally novel or grounded in decision-theory constructs not yet represented in the population. The \texttt{Explore} phase addresses this by leveraging the ontology to direct an LLM toward genuinely new regions of the model space.
Each explore round executes two complementary strategies in sequence. The \textbf{LLM-Guided} strategy (\texttt{llm\_guided}) presents the LLM with the current best-performing models and hard prediction cases, then asks it to propose new high-level behavioral concepts and functional-form templates that are (1) not yet represented and (2) potentially addresses the hard prediction cases. These enrichments expand the ontology's vocabulary but do not directly yield executable models. The \textbf{Ontology-Driven} strategy (\texttt{ontology\_driven}) follows immediately: it queries the ontology for underexplored functional forms and composed features (including those just proposed by \texttt{llm\_guided}), and asks the LLM to synthesize concrete symbolic expressions that utilize these components. The resulting model proposals are parsed, constant-optimized, and injected into the origin island's GP population. They are simultaneously cross-routed to all (schema-compatible) islands with numerical constants re-optimized on the destination data, providing structured knowledge transfer across experimental conditions.

\paragraph{Implementation details}
We use state-of-the-art {\tt Gemini 3.5 Flash} for steps involving LLMs, with {\tt temperature=0.7} to encourage LLM in-context exploration.

\section{Experiment on the {\tt Choices13k} Dataset}
\label{sec:append-experiment-13k}

\paragraph{Experiment Setup}
In this experiment we utilize the {\tt choices13k} \citep{petersen2021deep} dataset. The dataset collects around 50 human choices on each of the 8,931 choice problems without ambiguity, where the proportion is used to estimate choice probability.
We use log-likelihood, accuracy, and AIC/BIC to evaluate discovered models, and also compare discovered models with existing theories of human risky decision making.

\paragraph{Results}
Symbolic regression successfully discovered a rich set of models spanning different complexity (see \Cref{tab:pareto_notable_ablation13k2_nsga2_full}).
The pareto frontier captures the current best trade-off between prediction power and model complexity.
These models include re-discovered classical theories of risky choice, such as expected utility theory (complexity = 14) and prospect theory (complexity = 30).
The discovered propspect theory model achieves higher preference accuracy (82.4\%) than the neural propspect theory (82.33\%, \citep{peterson2021using}).
The best-performing model (complexity = 38) exhibits asymmetric comparison of prospects: it utilizes PT-style valuation with power transformation for both probability weighting and utility functions, but uses {\tt relu($\cdot$)} to focus on the advantage of each prospect for the final choice.

\begin{table}[ht]
\centering
\caption{Models on the discovered Pareto frontier on the Choices13k dataset. CE$_\mathrm{test}$ = test cross-entropy; MSE$_\mathrm{test}$ = test mean-squared error; Acc$_\mathrm{test}$ = test accuracy.}
\label{tab:pareto_notable_ablation13k2_nsga2_full}
\scriptsize
\setlength{\tabcolsep}{4pt}
\begin{tabular}{rrrr p{9cm}}
\toprule
Complexity & CE$_\mathrm{test}$ & MSE$_\mathrm{test}$ & Acc$_\mathrm{test}$ & Expression \\
\midrule
  2 & 0.6931 & 0.0488 & 0.504 & $-c_{1}$ \\
  6 & 0.6906 & 0.0475 & 0.541 & $-c_{1} \cdot \mathrm{dot}(\mathbf{p}_B, \mathbf{o}_B)$ \\
  8 & 0.6873 & 0.0459 & 0.594 & $-c_{1} \cdot \mathrm{dot}(\mathrm{sgp}(\mathbf{o}_B, c_{2}), \mathbf{p}_B)$ \\
  10 & 0.6457 & 0.0256 & 0.772 & $c_{1} \cdot \mathrm{dot}(\mathbf{p}_A, \mathbf{o}_A)$\newline\hspace*{1em}$- c_{2} \cdot \mathrm{dot}(\mathbf{p}_B, \mathbf{o}_B)$ \\
  12 & 0.6434 & 0.0247 & 0.773 & $-c_{1} \cdot \mathrm{sgp}(-\mathrm{dot}(\mathbf{o}_A, \mathbf{p}_A) + \mathrm{dot}(\mathbf{o}_B, \mathbf{p}_B), c_{2})$ \\
  14 & 0.6385 & 0.0221 & 0.813 & $c_{1} \cdot \mathrm{dot}(\mathbf{p}_A, \mathrm{sgp}(\mathbf{o}_A, c_{2}))$\newline\hspace*{1em}$- c_{3} \cdot \mathrm{dot}(\mathbf{p}_B, \mathrm{sgp}(\mathbf{o}_B, c_{4}))$ \\
  18 & 0.6355 & 0.0208 & 0.820 & $c_{1} \cdot \mathrm{dot}(\mathrm{sgp}(\mathbf{p}_A, c_{2}), \mathrm{sgp}(\mathbf{o}_A, c_{3}))$\newline\hspace*{1em}$- c_{4} \cdot \mathrm{dot}(\mathrm{sgp}(\mathbf{p}_B, c_{5}), \mathrm{sgp}(\mathbf{o}_B, c_{6}))$ \\
  30 & 0.6338 & 0.0200 & \textbf{0.824} & $c_{1} \cdot \mathrm{dot}(\exp(-c_{2} \cdot \mathrm{sgp}(-\ln(\mathbf{p}_A), c_{3})), \mathrm{sgp}(\mathbf{o}_A, c_{4}))$\newline\hspace*{1em}$- c_{5} \cdot \mathrm{dot}(\exp(-c_{6} \cdot \mathrm{sgp}(-\ln(\mathbf{p}_B), c_{7})), \mathrm{sgp}(\mathbf{o}_B, c_{8}))$ \\
  38 & \textbf{0.6312} & \textbf{0.0190} & 0.811 & $-c_{1} \cdot \mathrm{relu}(-\mathrm{dot}(\mathrm{sgp}(\mathbf{p}_A, c_{2}), \mathrm{sgp}(\mathbf{o}_A, c_{3})) + \mathrm{dot}(\mathrm{sgp}(\mathbf{p}_B, c_{4}), \mathrm{sgp}(\mathbf{o}_B, c_{5})))$\newline\hspace*{1em}$+ c_{6} \cdot \mathrm{relu}(\mathrm{dot}(\mathrm{sgp}(\mathbf{p}_A, c_{7}), \mathrm{sgp}(\mathbf{o}_A, c_{8})) - \mathrm{dot}(\mathrm{sgp}(\mathbf{p}_B, c_{9}), \mathrm{sgp}(\mathbf{o}_B, c_{10})))$ \\
\bottomrule
\end{tabular}
\end{table}

\begin{figure}[htbp]
    \centering
    \includegraphics[width=\linewidth]{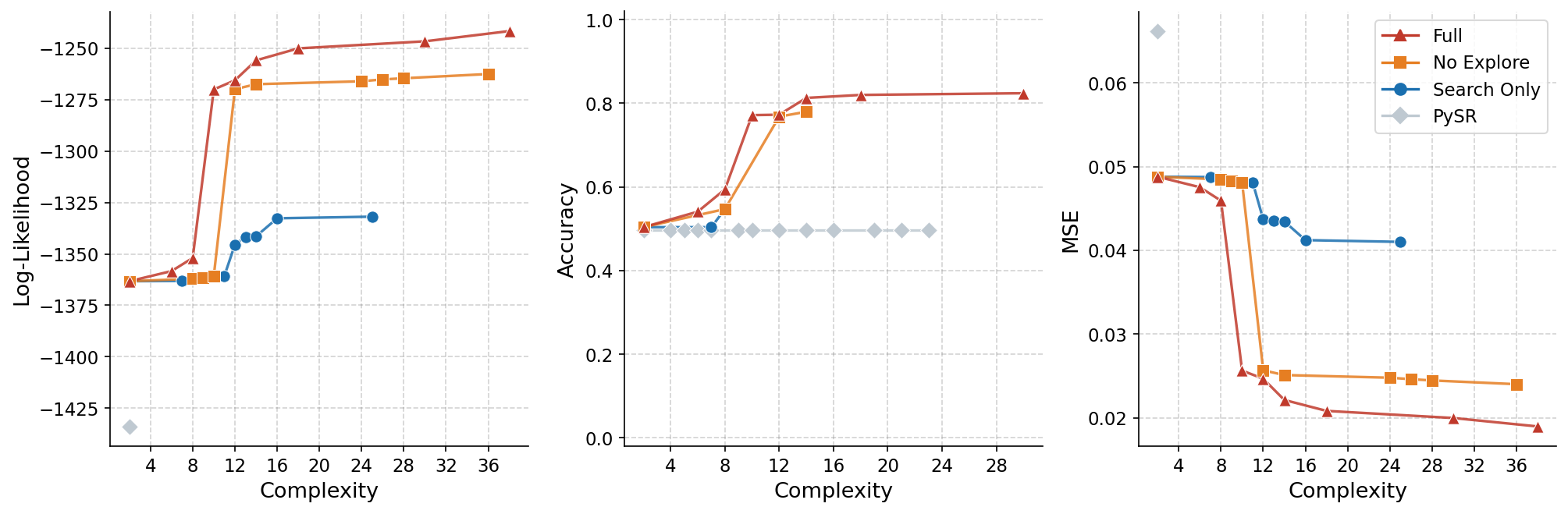}
    \caption{Pareto frontier of models on the {\tt choices13k} data.}
    \label{fig:exp-13k-pareto}
\end{figure}
\paragraph{Ablation}
We compare our framework (Full) with versions without {\sc Explore} (No Explore) and with {\sc Search} only (Search Only). 
We find {\tt Exploit} effectively boosts model performances on  {\sc Search}, while the full framework with {\sc Explore} achieves dominant performance (See \Cref{fig:exp-13k-pareto}).
In comparison, {\sc PySR} \citep{cranmer2023interpretable}, one of the state-of-the-art symbolic methods with unstructured matrix input has difficulty making meaningful discoveries with up to 36 input features.

\section{Risk Attitude in Combinatorial-Risk State Space}\label{sec:risk-attitude}

This section characterizes risk attitudes under combinatorial risk by studying the utility \(u\) on the space of \textit{probability states} rather than on the outcomes.
We first introduce a local directional risk premium and a corresponding local risk-aversion matrix, giving a multidimensional analogue of the Pratt measure.
We then derive tractable global classes of utility consistent with this local structure. In particular, under a rank one restriction, constant risk aversion delivers exponential-type rank-one forms and affine utility as the risk-neutral benchmark, which are the natural analogue of \citet{pratt1964risk}'s results.
% Under the same rank one restriction, decreasing risk aversion yields the more general representation \(u(m)=F(b^\top m)+\eta^\top m+D\), where all curvature is summarized by a single index \(b^\top m\).
These results connect local risk attitude to explicit choice patterns.
Proofs in this section are presented in the Appendix \Cref{sec:appendix-proof-risk-attitude}.

Let
\(\mathcal X=\{(0,0),(1,0),(0,1),(1,1)\}
\)
be the realized outcome space, and let
\(v:\mathcal X\to\mathbb R
\)
be utility over realized outcomes. Each probability state
\(m=(p_A,p_B)\in S:=[0,1]^2
\)
induces a lottery \(L(m)\) on \(\mathcal X\), with
\(
\Pr((1,1)\mid m)=p_Ap_B,\,
\Pr((1,0)\mid m)=p_A(1-p_B),\,
\Pr((0,1)\mid m)=(1-p_A)p_B,\,
\Pr((0,0)\mid m)=(1-p_A)(1-p_B).
\)
Let the initial state be \(m^0=(P_A^{\mathrm{ini}},P_B^{\mathrm{ini}}),
\)
and define the post-decision states
\[
m^A=(P_A^{\mathrm{ini}}+\Delta P_A,\;P_B^{\mathrm{ini}}),
\qquad
m^B=(P_A^{\mathrm{ini}},\;P_B^{\mathrm{ini}}+\Delta P_B).
\]
The induced value on the state space is
\(u(m):=\mathbb E_{L(m)}[v(X)].\)
The risk attitude is encoded in the geometry of \(u\), evaluated relative to the anchor \(m^0=(P_A^{\mathrm{ini}},P_B^{\mathrm{ini}})\).
Importantly, there is no literal risk over the state \(m=(p_A,p_B)\), and the perturbations introduced below are purely a tool to characterize this geometry of $u$.

Fix \(m\in S^\circ=(0,1)^2\), and let \(Y\) be a small mean-zero perturbation with \(m+Y\in S\) almost surely. Define the certainty equivalent set
\(
CE_u(m;Y):=\{c\in S:\ u(c)=\mathbb E[u(m+Y)]\},
\)
and the risk premium set
\(
RP_u(m;Y):=\{\pi\in\mathbb R^2:\ m-\pi\in S,\ u(m-\pi)=\mathbb E[u(m+Y)]\}.
\)
Assume \(u\in C^3(S^\circ)\), and write \(\Sigma_Y:=\mathbb E[YY^\top].
\)
Then
\[
\mathbb E[u(m+Y)]
=
u(m)+\frac12\tr\!\big(H_u(m)\Sigma_Y\big)+O(\mathbb E\|Y\|^3).
\]
If \(\pi\in RP_u(m;Y)\), then \(u(m-\pi)=\mathbb E[u(m+Y)],\)
hence
\begin{equation}
    \nabla u(m)^\top \pi
=
-\frac12\tr\!\big(H_u(m)\Sigma_Y\big)+o(\|Y\|^2).\label{eq:local-risk-premium-def}
\end{equation}
Notice that only the component of \(\pi\) in the direction of \(\nabla u(m)\) is identified.
To obtain a scalar premium, we impose a direction field \(d(m)\neq 0\), and let
\(\pi=\lambda\,d(m).
\)
Then, the premium is defined by \(u\bigl(m-\lambda d(m)\bigr)=\mathbb E[u(m+Y)].\)
If
\(\nabla u(m)^\top d(m)\neq 0,
\)
previous analysis motivates the local risk-attitude index
\begin{equation}
    \mathcal{A}(m) =  -\frac{H_u(m)}{\nabla u(m)^\top d(m)}.\label{eq:local-risk-attitude}
\end{equation}
Thus local risk aversion at $m$ means \(\mathcal A(m)\succeq 0\).

\subsection{Constant Risk Aversion}

Now we proceed to present the constant-risk-aversion
\[
-\frac{H_u(m)}{\nabla u(m)^\top d(m)}=M
\qquad\text{for all }m,
\]
for some constant symmetric matrix \(M\). Equivalently,
\(H_u(m)=\bigl(\nabla u(m)^\top d(m)\bigr)M.
\)

To derive an explicit $u(\cdot)$, we specialize to \(d(m)=(1,1)\). This restrict the risk premium set to the single direction that is indifferent between the dimensions. In the combinatorial risk setting, it means the decision maker is indifferent about \emph{which} outcome is successful, but instead aims to maximize total number of successes.
Then
\[
-\frac{H_u(m)}{\partial_1u(m)+\partial_2u(m)}=M,
\qquad\text{equivalently}\qquad
H_u(m)=-\bigl(\partial_1u(m)+\partial_2u(m)\bigr)M.
\]

\begin{proposition}[rank-one constant risk aversion]\label{prop:r1-const-risk-averse}
Let \(U\subseteq \mathbb R^2\) be connected, and \(u\in C^2(U)\) satisfies that
for all $m\in U$, \(\partial_1u(m)+\partial_2u(m)\neq 0\).
Define the rank-one constant matrix
\[
M=\frac{bb^\top}{b_1+b_2},
\qquad
b=(b_1,b_2)\in\mathbb R^2,
\qquad
b_1+b_2\neq 0.
\]
Then the local risk aversion is constant $M$
\[
-\frac{H_u(m)}{\partial_1u(m)+\partial_2u(m)}=M
\qquad\text{for all }m\in U,
\]
if and only if \(u\) is of the form
\begin{equation}
    u(m)=A e^{-b^\top m}+\eta^\top m+D, \label{eq:utility-const-risk-averse}
\end{equation}
for some constants \(A,D\in\mathbb R\) and \(\eta\in\mathbb R^2\) with $\eta^\top(1,1)=0$.
\end{proposition}

\begin{remark}
The functional form in the rank-one case is the natural multidimensional analogue of \citet{pratt1964risk}'s exponential utility. In the one-dimensional case, constant risk aversion implies an exponential value function because curvature is everywhere proportional to marginal utility. Here the same logic applies, but only along the single rank-one direction selected by \(M\). Since
\[
M=\frac{bb^\top}{b_1+b_2}
\]
has image \(\mathrm{span}\{b\}\), all second-order curvature is confined to the one-dimensional index \(b^\top m\): the Hessian is proportional to \(bb^\top\), so \(u\) can bend only along \(b\), and is locally flat in directions orthogonal to \(b\) up to affine terms. This is why the nonlinear part of \(u\) depends only on \(b^\top m\). The remaining term \(\eta^\top m + D\) is affine, with the restriction \(\eta^\top(1,1)=0\) ensuring that it does not affect the normalization by \(\partial_1u+\partial_2u\). Thus the exponential component captures constant risk sensitivity along the effective one-dimensional risk index \(b^\top m\), while the affine component reflects directions that are invisible to the \(d=(1,1)\) normalization.
\end{remark}

\begin{corollary}[Choice under rank-one constant risk aversion]\label{cor:r1-const-risk-averse}
Let
\(z^{\mathrm{ini}}:=b_1P_A^{\mathrm{ini}}+b_2P_B^{\mathrm{ini}}.
\)
Suppose
\[
u(m)=A e^{-b^\top m}+\eta^\top m + D,
\qquad
b=(b_1,b_2),\qquad
\eta^\top(1,1)=0,
\]
as in Proposition \ref{prop:r1-const-risk-averse}. 
Then the decision-maker prefers \(A\) over \(B\) if and only if
\[
A e^{-z^{\mathrm{ini}}}
\bigl(e^{-b_1\Delta P_A}-e^{-b_2\Delta P_B}\bigr)
+\eta_1(\Delta P_A+\Delta P_B)\ge 0.
\]
\end{corollary}

\begin{proposition}[Risk-neutrality]\label{prop:risk-neutral}
Let \(U\subseteq \mathbb R^2\) be connected, and \(u\in C^2(U)\) satisfies that
for all $m\in U$, \(\partial_1u(m)+\partial_2u(m)\neq 0\).
Then the decision-maker is risk neutural
\[
\frac{H_u(m)}{\partial_1 u(m)+\partial_2 u(m)}=0
\qquad\text{for all }m\in S^\circ
\]
if and only if \(u\) is affine:
\[
u(m)=a+\beta^\top m
\]
for some \(a\in\mathbb R\) and \(\beta\in\mathbb R^2\). Conversely, every affine function satisfies \(H_u\equiv 0\), hence solves the equation with \(M=0\).
\end{proposition}

\begin{remark}
Under the normalization \(d=(1,1)\), the directional premium is locally well defined only at states \(m\) such that
\[
\partial_1 u(m)+\partial_2 u(m)\neq 0.
\]
In particular, if \(u(m)=a+\beta^\top m\), this condition reduces to
\(\beta_1+\beta_2\neq 0.
\)
If instead \(\beta_1+\beta_2=0\), then \(d=(1,1)\) is tangent to the indifference lines of \(u\), so the normalization is degenerate.
\end{remark}

Under risk-neutrality, the initial probabilities
\(P_A^{\mathrm{ini}},P_B^{\mathrm{ini}}\) cancel out, so choice depends only on the weighted probability increments.

\begin{corollary}[Choice under risk-neutrality]\label{cor:risk-neutral}
Suppose
\[
u(m)=a+\beta^\top m
\qquad\text{with}\qquad
\beta=(\beta_1,\beta_2),
\]
as in Proposition \ref{prop:risk-neutral}. Then the decision-maker prefers \(A\) over \(B\) if and only if
\[
\beta_1\Delta P_A- \beta_2\Delta P_B \ge 0.
\]
\end{corollary}

\subsection{Decreasing Risk Aversion}

We now extend the preceding Pratt-style analysis from constant to decreasing risk aversion.
Fix the normalization $d(m)=(1,1),$
so that the directional risk premium is measured along the common shift of both coordinates.
For a small mean-zero perturbation \(Y\) with covariance matrix
\(\Sigma_Y:=\mathbb E[YY^\top],
\)
the local premium \(\lambda=\lambda(m;Y)\) defined by
\[
u\bigl(m-\lambda(1,1)\bigr)=\mathbb E[u(m+Y)]
\]
satisfies
\[
\lambda(m;Y)
=
-\frac12\,
\frac{\tr\!\big(H_u(m)\Sigma_Y\big)}
{\partial_1u(m)+\partial_2u(m)}
+o(\|Y\|^2).
\]
Accordingly, define the local risk-aversion matrix
\begin{equation}
\mathcal A(m):=
-\frac{H_u(m)}{\partial_1u(m)+\partial_2u(m)},
\qquad
\partial_1u(m)+\partial_2u(m)\neq 0.
\label{eq:DRA-local-index}
\end{equation}
Then
\[
\lambda(m;Y)
=
\frac12\tr\!\big(\mathcal A(m)\Sigma_Y\big)+o(\|Y\|^2).
\]

Local risk aversion means \(\mathcal A(m)\succeq 0\), equivalently \(H_u(m)\preceq 0\) when \(\partial_1u(m)+\partial_2u(m)>0\). The natural analogue of decreasing absolute risk aversion is that this matrix decreases as the current status increases.

\begin{definition}[Decreasing risk aversion]
Let \(U\subseteq \mathbb R^2\) be connected, and let \(u\in C^2(U)\) satisfy \(\partial_1u(m)+\partial_2u(m)>0
\)
for all $m\in U$.
We say that \(u\) exhibits \emph{decreasing risk aversion} on \(U\) if

\begin{enumerate}
    \item for every \(m\in U\), \(\mathcal A(m)\succeq 0,\)
    so every mean-zero risk carries a nonnegative local premium;
    \item whenever \(m',m\in U\) satisfy \(m'\ge m\) componentwise,
    \(
    \mathcal A(m')\preceq \mathcal A(m).
    \)
\end{enumerate}
\end{definition}

A tractable solution class again arises in the rank-one case. Fix \(b=(b_1,b_2)\) with \(b_1+b_2\neq 0\), and suppose the local risk-aversion matrix takes the form
\begin{equation}
\mathcal A(m)
=
\rho(b^\top m)\,\frac{bb^\top}{b_1+b_2},
\label{eq:DRA-rank-one-index}
\end{equation}
where \(\rho:\mathbb R\to\mathbb R\) is a scalar function. Then risk aversion corresponds to \(\rho\ge 0\), and decreasing risk aversion corresponds to \(\rho\) being weakly decreasing along the index \(b^\top m\).

\begin{proposition}[rank-one decreasing risk aversion]\label{prop:r1-decrease-risk-averse}
Let \(U\subseteq\mathbb R^2\) be connected, and let \(u\in C^2(U)\) satisfy \(\partial_1u(m)+\partial_2u(m)>0
\)
for all $m\in U$.
Fix \(b=(b_1,b_2)\in\mathbb R^2\) with \(b_1+b_2\neq 0\). Then
\[
-\frac{H_u(m)}{\partial_1u(m)+\partial_2u(m)}
=
\rho(b^\top m)\,\frac{bb^\top}{b_1+b_2}
\qquad\text{for all }m\in U
\]
if and only if \(u\) is of the form
\begin{equation}
u(m)=F(b^\top m)+\eta^\top m + D,
\label{eq:DRA-rank-one-utility}
\end{equation}
for some \(F\in C^2\), \(D\in\mathbb R\), and \(\eta\in\mathbb R^2\) satisfying $\eta^\top(1,1)=0,$
with
\begin{equation}
-\frac{F''(z)}{F'(z)}=\rho(z)
\qquad\text{for all }z=b^\top m,\ m\in U.
\label{eq:DRA-one-dimensional-ode}
\end{equation}
\end{proposition}

\begin{remark}
The rank-one form shows that decreasing risk aversion remains effectively one-dimensional. The nonlinear part of utility depends only on the index \(b^\top m\), and the scalar function
\[
\rho(z)=-\frac{F''(z)}{F'(z)}
\]
is the exact analogue of \citet{pratt1964risk}'s absolute risk-aversion coefficient. Risk aversion requires \(\rho(z)\ge 0\), while decreasing risk aversion requires \(\rho\) to be weakly decreasing in \(z\). The constant-risk-aversion case is recovered when \(\rho\) is constant, in which case \(F\) is exponential.
\end{remark}

\begin{remark}
    The constant-risk-aversion case is recovered by taking
\(
F(z)=A e^{-z}.
\)
\end{remark}

\begin{corollary}\label{cor:r1-linear-decrease-risk-averse}
Under the conditions of Proposition \ref{prop:r1-decrease-risk-averse}, suppose
\[
\rho(w)=-kw,
\qquad k\ge 0.
\]
Then 
\[
F'(w)=C e^{\frac{k}{2}w^2}
\]
for some constant \(C\in\mathbb R\), and therefore
\[
F(w)=C\int_0^w e^{\frac{k}{2}t^2}\,dt + D
\]
for some constants \(C,D\in\mathbb R\).
\end{corollary}

\begin{corollary}\label{cor:r1-exp-decrease-risk-averse}
Under the conditions of Proposition \ref{prop:r1-decrease-risk-averse}, suppose
\[
\rho(w)=e^{-\gamma w},
\qquad \gamma>0.
\]
Then
\[
F'(w)=C\exp\!\Bigl(\frac{1}{\gamma}e^{-\gamma w}\Bigr)
\]
for some constant \(C\in\mathbb R\), and therefore
\[
F(w)=C\int_0^w \exp\!\Bigl(\frac{1}{\gamma}e^{-\gamma t}\Bigr)\,dt + D
\]
for some constants \(C,D\in\mathbb R\).
\end{corollary}

\begin{corollary}[Choice under rank-one decreasing risk aversion]\label{cor:r1-decrease-risk-averse}
Let
\(
z^{\mathrm{ini}}:=b_1P_A^{\mathrm{ini}}+b_2P_B^{\mathrm{ini}}.
\)
Suppose
\[
u(m)=F(b^\top m)+\eta^\top m + D,
\qquad
b=(b_1,b_2),\qquad
\eta^\top(1,1)=0,
\]
as in Proposition \ref{prop:r1-decrease-risk-averse}. 
Then the decision-maker prefers \(A\) over \(B\) if and only if
\[
F(z^{\mathrm{ini}}+b_1\Delta P_A)-F(z^{\mathrm{ini}}+b_2\Delta P_B)
+\eta_1(\Delta P_A+\Delta P_B)\ge 0.
\]
\end{corollary}

\section{Proofs in \Cref{sec:risk-attitude}}\label{sec:appendix-proof-risk-attitude}
\subsection{Proof for \Cref{prop:r1-const-risk-averse}}
\begin{proof}
Write \(s(m):=\partial_1u(m)+\partial_2u(m).\)
The equation is equivalent to
\[
H_u(m)=-s(m)\,\frac{bb^\top}{b_1+b_2}.
\]
Let \(u_{11}=-\dfrac{b_1^2}{b_1+b_2}s,\,
u_{12}=-\dfrac{b_1b_2}{b_1+b_2}s,\,
u_{22}=-\dfrac{b_2^2}{b_1+b_2}s.\)
Summing the first two and the last two gives
\[
\partial_1 s=u_{11}+u_{21}=u_{11}+u_{12}=-b_1 s,
\]
\[
\partial_2 s=u_{12}+u_{22}=-b_2 s.
\]
Thus
\[
\nabla s=-s\,b.
\]
Since \(U\) is connected, it follows that
\[
s(m)=C e^{-b^\top m}
\]
for some constant \(C\in\mathbb R\).
Substituting back,
\[
H_u(m)=\frac{C}{b_1+b_2} e^{-b^\top m}bb^\top.
\]
Let
\[
A:=\frac{C}{b_1+b_2}.
\]
Then
\[
H_u(m)=A e^{-b^\top m}bb^\top
=H\!\left(A e^{-b^\top m}\right).
\]
Therefore
\[
H\!\left(u-A e^{-b^\top m}\right)=0.
\]
Since \(U\) is connected, \(u-A e^{b^\top m}\) must be affine, so
\[
u(m)=A e^{-b^\top m}+\eta^\top m + D
\]
for some \(\eta\in\mathbb R^2\) and \(D\in\mathbb R\).

Finally,
\[
\partial_1u+\partial_2u
=A e^{-b^\top m}(b_1+b_2)+\eta^\top(1,1).
\]
But this must equal
\[
s(m)=C e^{-b^\top m}=A(b_1+b_2)e^{-b^\top m},
\]
so necessarily
\[
\eta^\top(1,1)=0.
\]
This proves the characterization. The converse was obvious.
\end{proof}

\subsection{Proof for \Cref{cor:r1-const-risk-averse}}
\begin{proof}
\[
u(m^A)=A e^{-(z^{\mathrm{ini}}+b_1\Delta P_A)}
+\eta_1(P_A^{\mathrm{ini}}+\Delta P_A)+\eta_2P_B^{\mathrm{ini}}+D,
\]
and
\[
u(m^B)=A e^{-(z^{\mathrm{ini}}+b_2\Delta P_B)}
+\eta_1P_A^{\mathrm{ini}}+\eta_2(P_B^{\mathrm{ini}}+\Delta P_B)+D.
\]
Hence
\[
u(m^A)-u(m^B)
=
A e^{-z^{\mathrm{ini}}}
\bigl(e^{-b_1\Delta P_A}-e^{-b_2\Delta P_B}\bigr)
+\eta_1\Delta P_A-\eta_2\Delta P_B.
\]
Using \(\eta^\top(1,1)=0\), i.e. \(\eta_2=-\eta_1\), this becomes
\[
u(m^A)-u(m^B)
=
A e^{-z^{\mathrm{ini}}}
\bigl(e^{-b_1\Delta P_A}-e^{-b_2\Delta P_B}\bigr)
+\eta_1(\Delta P_A+\Delta P_B).
\]
\end{proof}

\subsection{Proof for \Cref{prop:risk-neutral}}
\begin{proof}
If the ratio is identically zero, then necessarily
\[
H_u(m)=0
\qquad\text{for all }m\in S^\circ.
\]
A \(C^2\) function with vanishing Hessian on a connected open set is affine, so
\[
u(m)=a+\beta^\top m.
\]
Conversely, if \(u(m)=a+\beta^\top m\), then \(H_u(m)=0\) for all \(m\), and therefore the constant-risk-aversion equation with \(M=0\) is satisfied.
\end{proof}

\subsection{Proof for \Cref{cor:risk-neutral}}
\begin{proof}
\[
    u(m^A)=a+\beta_1(P_A^{\mathrm{ini}}+\Delta P_A)+\beta_2 P_B^{\mathrm{ini}},
    \]
    and
    \[
    u(m^B)=a+\beta_1P_A^{\mathrm{ini}}+\beta_2(P_B^{\mathrm{ini}}+\Delta P_B).
    \]
    Therefore
    \[
    u(m^A)-u(m^B)=\beta_1\Delta P_A-\beta_2\Delta P_B.
\]
\end{proof}

\subsection{Proof for \Cref{prop:r1-decrease-risk-averse}}
\begin{proof}
Let
\[
s(m):=\partial_1u(m)+\partial_2u(m).
\]
The displayed equation is equivalent to
\[
H_u(m)
=
-\,s(m)\,\rho(b^\top m)\,\frac{bb^\top}{b_1+b_2}.
\]
Hence
\[
u_{11}= -\rho(b^\top m)\frac{b_1^2}{b_1+b_2}s,\qquad
u_{12}= -\rho(b^\top m)\frac{b_1b_2}{b_1+b_2}s,\qquad
u_{22}= -\rho(b^\top m)\frac{b_2^2}{b_1+b_2}s.
\]
Summing gives
\[
\partial_1 s=u_{11}+u_{12}=-\rho(b^\top m)b_1 s,
\qquad
\partial_2 s=u_{12}+u_{22}=-\rho(b^\top m)b_2 s.
\]
Therefore
\[
\nabla s = -\rho(b^\top m)\,s\,b.
\]
It follows that \(s\) depends only on \(z=b^\top m\), say \(s(m)=\sigma(z)\), where
\[
\sigma'(z)=-\rho(z)\sigma(z).
\]
Choose \(F\) so that
\[
F'(z)=\frac{\sigma(z)}{b_1+b_2}.
\]
Then
\[
F''(z)=-\rho(z)F'(z),
\]
which is exactly \eqref{eq:DRA-one-dimensional-ode}. Moreover,
\[
H\!\big(F(b^\top m)\big)=F''(b^\top m)\,bb^\top=H_u(m),
\]
so \(u-F(b^\top m)\) has zero Hessian on the connected set \(U\), hence is affine:
\[
u(m)=F(b^\top m)+\eta^\top m + D.
\]
Finally,
\[
\partial_1u(m)+\partial_2u(m)
=
(b_1+b_2)F'(b^\top m)+\eta^\top(1,1)
=
s(m),
\]
so necessarily
\[
\eta^\top(1,1)=0.
\]
The converse is immediate by direct differentiation.
\end{proof}

\subsection{Proof for \Cref{cor:r1-linear-decrease-risk-averse}}
\begin{proof}
From Proposition \ref{prop:r1-decrease-risk-averse},
\[
-\frac{F''(w)}{F'(w)}=\rho(w)=-kw,
\]
so
\[
\frac{d}{dw}\log F'(w)=kw.
\]
Integrating gives
\[
\log F'(w)=\frac{k}{2}w^2+\log C,
\]
hence
\[
F'(w)=C e^{\frac{k}{2}w^2}.
\]
Integrating once more yields
\[
F(w)=C\int_0^w e^{\frac{k}{2}t^2}\,dt + D.
\]
\end{proof}

\subsection{Proof for \Cref{cor:r1-exp-decrease-risk-averse}}
\begin{proof}
From Proposition \ref{prop:r1-decrease-risk-averse},
\[
-\frac{F''(w)}{F'(w)}=\rho(w)=e^{-\gamma w},
\]
so
\[
\frac{d}{dw}\log F'(w)=-e^{-\gamma w}.
\]
Integrating gives
\[
\log F'(w)=\frac{1}{\gamma}e^{-\gamma w}+\log C,
\]
hence
\[
F'(w)=C\exp\!\Bigl(\frac{1}{\gamma}e^{-\gamma w}\Bigr).
\]
Integrating once more yields
\[
F(w)=C\int_0^w \exp\!\Bigl(\frac{1}{\gamma}e^{-\gamma t}\Bigr)\,dt + D.
\]
\end{proof}

\subsection{Proof for \Cref{cor:r1-decrease-risk-averse}}
\begin{proof}
\[
u(m^A)=F(z^{\mathrm{ini}}+b_1\Delta P_A)
+\eta_1(P_A^{\mathrm{ini}}+\Delta P_A)+\eta_2P_B^{\mathrm{ini}}+D,
\]
and
\[
u(m^B)=F(z^{\mathrm{ini}}+b_2\Delta P_B)
+\eta_1P_A^{\mathrm{ini}}+\eta_2(P_B^{\mathrm{ini}}+\Delta P_B)+D.
\]
Therefore
\[
u(m^A)-u(m^B)
=
F(z^{\mathrm{ini}}+b_1\Delta P_A)-F(z^{\mathrm{ini}}+b_2\Delta P_B)
+\eta_1\Delta P_A-\eta_2\Delta P_B.
\]
Using \(\eta_2=-\eta_1\), this simplifies to
\[
u(m^A)-u(m^B)
=
F(z^{\mathrm{ini}}+b_1\Delta P_A)-F(z^{\mathrm{ini}}+b_2\Delta P_B)
+\eta_1(\Delta P_A+\Delta P_B).
\]
\end{proof}

\end{appendices}

\end{document}